\documentclass[a4paper]{article}%
\usepackage[latin1]{inputenc}
\usepackage{amsmath,amssymb}
\usepackage{amsfonts}
\usepackage{amsthm}
\usepackage{graphicx}
\usepackage{url}
\usepackage{media9}
\usepackage{multirow}
\usepackage{bm}
\usepackage{chicago}
\usepackage{color}
\usepackage{amsmath}
\usepackage{amssymb}%
\providecommand{\U}[1]{\protect\rule{.1in}{.1in}}
\providecommand{\U}[1]{\protect\rule{.1in}{.1in}}
\providecommand{\U}[1]{\protect\rule{.1in}{.1in}}
\graphicspath{
{R_scripts/}
{Datasets/Idealista.com/}
{../paper_1/}
}

\def\comment#1{}
\newcommand{\PP}{ \mathbb{P} }

\newcommand{\bz}{ \mathbf{z} }

\newcommand{\bA}{ \mathbf{A} }

\newcommand{\ba}{ \mathbf{a} }
\newcommand{\bb}{ \mathbf{b} }

\newcommand{\bV}{ \mathbf{V} }

\newcommand{\bi}{\begin{itemize}}
\newcommand{\ei}{\end{itemize}}

\newcommand{\bc}{\begin{center}}
\newcommand{\ec}{\end{center}}

\def\citet{\citeN}
\def\bi{\begin{itemize}}
\def\ei{\end{itemize}}

\newtheorem{lemma}{Lemma}
\newtheorem{proposition}{Proposition}
\newtheorem{theorem}{Theorem}
\newtheorem{corol}{Corollary}

\begin{document}

\title{Understanding complex predictive models with Ghost Variables}
\author{Pedro Delicado\\{\small Departament d'Estad\'{\i}stica i Investigaci\'on Operativa} \\{\small Universitat Polit\`ecnica de Catalunya} \\\ \\Daniel Pe\~na \\{\small Departamento de Estad\'{\i}stica and Institute of Financial Big Data} \\{\small Universidad Carlos III de Madrid} }
\date{August 27, 2019} 
\maketitle

\begin{abstract}
We propose a procedure for assigning a relevance measure to each 
explanatory variable in a complex predictive model. We assume that we have a
training set to fit the model and a test set to check the out of sample
performance. First, the individual relevance of each variable is computed by
comparing in the test set the predictions of  the model that includes all the variables, 
with those of another model in which the variable of interest is substituted by its ghost
variable, defined as the prediction of this variable by using the rest of
explanatory variables. Second, we check the joint relevance of the variables
by using the eigenvalues of a relevance matrix that is the covariance matrix
of the vectors of individual effects. It is shown that in simple models, as
linear or additive models, the proposed measures are related to standard
measures of significance of the variables and in neural networks models (and
in other algorithmic prediction models) the procedure provides information
about the joint and individual effects of the variables that is not usually
available by other methods. The procedure is illustrated with simulated examples and the
analysis of a large real data set.


\end{abstract}

\section{Introduction}

\label{sec:intro}

In a stimulating and provocative paper, \citeN{Breiman:2001} shook the
statistical community by arguing that traditional Statistics was no longer the
only way to reach conclusions from data. He noted that, in addition to the
\emph{Data Modeling Culture} (traditional Statistics), was emerging an
\emph{Algorithmic Modeling Culture}.
At that time Breiman was almost surely thinking on Machine Learning, and now
we could also think in Data Science. The algorithms proposed by the new
culture (neural networks (NN), support vector machines, random forest, etc.)
have often better predictive accuracy than traditional statistical models
(linear regression, logistic regression, etc.). On the other hand, statistical
models explain better the relationship among response and input variables. In
fact, these new-culture algorithms usually are called \emph{black boxes.}

The need of understanding the variables effects increases with prediction
rules based on Big Data, that is, data sets with large number of variables,
$p,$ and observations, $n,$ and even with $p>n$, because: (1) in some models,
as NN, the effect of a variable is a non linear function of its weights in
different linear combinations, making its total effect difficult to see; (2)
the effect of a variable strongly correlated with others cannot be well
understood without considering their joint effects. The crucial importance of
understanding the effect of the variables on algorithm models has often been
recognized, for instance, \citeN{ribeiro2016should} state that \emph{a vital
concern remains: if the users do not trust a model or a prediction, they will
not use it.} There is a powerful research line on \emph{variable importance
measure in algorithmic models}, which main goal is to provide interpretability
tools lighting these black box models, that has been labeled as
\emph{eXplainable Artificial Intelligence} (XAI) (e.g.,
\shortciteNP{baehrens2010explain},
\citeNP{kononenko2010efficient},
\citeNP{Biran2017ExplanationAJ}, \citeNP{Nott:2017}, \citeNP{Biecek:2018},
\citeNP{Delicado:2019}, and \citeNP{Miller:2019})
or {\em Interpretable Machine Learning} (IML) (e.g., \shortciteNP{IML_package:2018}, \citeNP{Molnar:2019}, 
\shortciteNP{Murdoch_et_al:2019},
and \shortciteNP{VIP_package:2019}).
An often used approach in neural networks (NN) is to look at the derivatives of the prediction function
with respect to each variable. See \citeN{intrator2001interpreting},
that introduced the generalized weights in NN, and
\shortciteN{montavon2018methods}
for a survey of this field. However, these procedures have limitations to show
the general pattern of interactions among variables.

We consider in this article the general framework of a prediction problem
involving the random vector $(X,Z,Y)$, $X\in\mathbb{R}^{p}$, $Z\in\mathbb{R}$
and $Y\in\mathbb{R}$, where $Y$ is the response variable that should be
predicted from $(X,Z)$. It is assumed that there exist a training sample of
size $n_{1}$ (used to fit the prediction rule) and a test sample of size
$n_{2}$ (used for checking the out of sample precision accuracy of the fitted
prediction rule). A simple approach to define the importance of the variable
$Z$ consists in measuring its contribution to the forecast of the response. To
compute this measure the model is fitted twice: first including both $X$ and
$Z$ and then excluding $Z$. This approach is often used to decide if the
variable $Z$ should be included in the model. However, other alternatives are
possible. 
Instead of deleting the variable $Z$, \citeN{Breiman:2001} proposed
to randomly permute its values in the test sample to create a new variable
$Z^{\prime}$
and compare the model predictions using the original $Z$ variable and the
randomized one $Z^{\prime}$.
This method has two main advantages over deleting the variable: (1) only one
predictive model has to be fitted (the one using all the explanatory
variables); 
and (2) the forecast comparison is made between two models with the same number of
variables. A drawback is that $Z^{\prime}$
is unrelated to $Y$ but also to $X$ and the joint effect of the $Z$ variable
with the $X$ will be lost.
It is worth mentioning that both methods, deleting $Z$ and replacing it by $Z'$,
are \emph{model agnostic} (\shortciteNP{ribeiro2016model}) because they can be applied to any predictive model,
which is a desirable property.
Despite its popularity, using random permutations for interpreting black box prediction algorithms has received numerous criticisms.  
See, for instance, \citeN{HookMent:2019:StopPermuting} who advocate against this practice. 

In this article we propose a third approach to measure the effect of a
variable that combines the advantages of deleting and randomizing, and at the same time avoids their drawbacks. 
We fit the model with all the original variables,
$X$ and $Z$,
in the training sample. Then we apply the model in the test sample twice:
first using the observed values of $Z$ and then using a new variable
$\widehat{Z}$
that is uncorrelated to the response given the variables $X$,
(as in the randomized method) but has the same relation with the $X$ variables
as the original $Z$ variable (opposite to the randomized variable). We call
this new variable a \emph{ghost variable}, and it can be computed as $\hat
{Z}=\widehat{\mathbb{E}(Z|X)}$, the estimated expected value of $Z$ given the
other explanatory variables $X$. 
Working with ghost variables has certain similarities with the use of conditional permutations, as proposed by \shortciteN{strobl2008conditional} for random trees, 
and with the simulation of data from the conditional distribution $Z|X$
suggested by \citeN{HookMent:2019:StopPermuting}.
Finally, we extend our proposals to measure the relationship among groups of variables, showing that analyzing the relationships between the effects of the variables, one by one, helps to identify the most important groups of variables in prediction.

Our objective is to understand the effect of the variables in a complex
prediction model and not to select the variables to be included in a linear
model. However, both problems are related and some of 
our results could be useful for variable selection in regression. 
For instance, 
\citeN{wu2007controlling}
proposed controlling variable selection in regression by
adding pseudovariables, defined as variables unrelated to the response and
also to the explanatory variables. We also add ghost variables to the model but
they are very different to the pseudovariables. Both are unrelated to the
response but the pseudo variables are also unrelated to the explanatory variables
whereas the ghost variables 
keep the correlation structure of the explanatory variables. 
\citeANP{barber2015controlling} \citeyear{barber2015controlling,barber2019knockoff}
and 
\shortciteN{candes2018panning}
introduced Knockoffs to control the false discovery rate in large regression
problems. Their knockoffs variables are all added jointly in the model
and have a different objective than our ghost variables. 
However,  both are unrelated to the response and try to keep the correlation structure of the explanatory variables.
In fact, the analysis that we propose can be useful as a screening device  to understand the
dependency among the variables in the
large problems considered by 
\citeN{barber2019knockoff}.

The rest of the article is organized as follows. Section \ref{sec:VarRelPopul}
presents the problem of measuring the relevance of a variable in a prediction
problem at the population level. We review briefly the two main approaches
usually used, delete the variable or substitute it by a random permutation,
and proposed a new way: to substitute it by its ghost variable defined as its
expected value given the rest of the explanatory variables. Section
\ref{sec:VarRelSample} deals with the sample implementation of these variable
relevance measures, and compares them when they are applied to simple models,
as additive models or multiple linear regression. We show that they produce
reasonable results in these well-known settings and they are strongly related
to usual significance statistics. These support the generalization to more
complex models. Section \ref{sec:Rel.Matrix} defines the relevance
matrix as the covariance matrix of the individual effects and shows that the
eigenvectors of this matrix will have useful information about the joint
effects of the
variables.
It is also shown that, in linear regression, this matrix is closely related with the partial correlation matrix.
Section \ref{sec:examples} illustrates the properties of the relevance matrix in simulated examples and analyzing a real data set.
Section \ref{sec:concl} concludes with some suggestions for further research.
The proofs of the results in Sections \ref{sec:VarRelSample} and
\ref{sec:Rel.Matrix} are deferred to the appendixes. An Online Supplement
includes a part of the output corresponding to the real data example.

\textbf{Remark}. The term ghost variable has been used in computer science to
refer to auxiliary variables that appear in a program to facilitate
specification and verification, for instance, to count the number of
iterations of a loop. Once the program is running correctly they can be
eliminated from the program.\ We believe that this name has not been used
before in Statistics although, of course, the missing value approach to
identify the effect of data points or variables has a long and successful
tradition to identify outliers, influential observations and sensitive points 
in time series and multivariate analysis.

\section{Relevance for random variables}

\label{sec:VarRelPopul} Let us consider the prediction problem involving the
random vector $(X,Z,Y)$, $X\in\mathbb{R}^{p}$, $Z\in\mathbb{R}$ and
$Y\in\mathbb{R}$, where $Y$ is the response variable that should be predicted
from $(X,Z)$. A \emph{prediction function} $f:\mathbb{R}^{p+1}\rightarrow
\mathbb{R}$ has \emph{expected loss} (or \emph{risk}) $R(f(X,Z),Y)=\mathbb{E}%
(L(f(X,Z),Y))$, where $L:\mathbb{R}\times\mathbb{R}\rightarrow\mathbb{R}^{+}$
is a loss function measuring the \emph{cost} associated with predicting $Y$ by
$f(X,Z)$, in particular quadratic loss is $L(y,y^{\prime})=(y-y^{\prime})^{2}$.

We consider the problem of measuring the effect of $Z$ on the $f$ when
predicting $Y$ by $f(X,Z)$. Also, we can compute a reduced version of $f$, say
$f_{p}(X).$ A usual way to check for the relevance of the variable $Z$ is to
compute $R(f_{p}(X),Y)-R(f(X,Z),Y).$ An alternative way is to compute the
change in the prediction function $\mathbb{E}(L(f(X,Z),f_{p}(X))).$ These two
relevance measures do not necessarily coincide, but they do in a relevant
case: under quadratic loss and
assuming $Y=f(X,Z)+\varepsilon$, as
\[
R(f_{p}(X),Y)=\mathbb{E}((Y-f_{p}(X))^{2})=\mathbb{E}%
(\{(Y-f(X,Z))+(f(X,Z)-f_{p}(X))\}^{2})=
\]%
\[
R(f(X,Z),Y)+\mathbb{E}(L(f(X,Z),f_{p}(X))).
\]
Both measures evaluate the effect of the single variable $Z$ by
\emph{suppressing} it from the prediction function $f$. We choose to use
$\mathbb{E}(L(f(X,Z),f_{p}(X)))$ because it does not depend on the
noisy part of the response variable $Y$.

An alternative approach, that does not require $f_{p}$, is to replace $Z$ by
an independent copy $Z^{\prime}$: a random variable with the same marginal
distribution as $Z$ but independent from $(X,Y),$ and define the relevance of
$Z$ by
$\mathbb{E}(L(f(X,Z),f(X,Z^{\prime})))$. 
However, this approach has some
limitations. Consider the case of $f$ being additive in $X$ and $Z$:
$f(X,Z)=s_{1}(X)+s_{2}(Z)$, with $\mathbb{E}(s_{2}(Z))=0$. Under quadratic
loss,
\[
\mathbb{E}(L(f(X,Z),f(X,Z^{\prime})))=\mathbb{E}\left\{  (s_{1}(X)+s_{2}%
(Z))-(s_{1}(X)+s_{2}(Z^{\prime}))\right\}  ^{2}=2\mbox{Var}(s_{2}(Z)).
\]
If additional linearity happens, $s_{2}(Z)=Z\beta_{z}$, then this relevance is
$2\beta_{z}^{2}\mbox{Var}(Z)$. 
These results are similar to those obtained by \shortciteN{zhu2015reinforcement} or \shortciteN{Gregorutti_et_al:2017}.
At a first glance these relevance measures
$(2\mbox{Var}(s_{2}(Z))$ or $2\beta_{z}^{2}\mbox{Var}(Z)$) seem to be
suitable, but:

\begin{itemize}
\item[(1)] The relevance of $Z$ would be the same in one case where $X$ and
$Z$ are independent (so $Z$ encode exclusive information about $Y$) and in
another case where $X$ and $Z$ are strongly related (in such a case $X$ could
make up for the absence of $Z$). Clearly $Z$ is more relevant in the first
case than in the second one.

\item[(2)] The replacement of $Z$ by an independent copy $Z^{\prime}$ implies
an alteration of the prediction function $f(X,Z)$. Consider again the simple
case of the linear predictor $f(X,Z)=\beta_{0}+X^{T}\beta_{x}+Z\beta_{z}$.
Replacing $Z$ by $Z^{\prime}$ is equivalent to using the following reduced
version of $f$:
\[
f_{p}(X)=\beta_{0}^{\prime}+X^{T}\beta_{x}+\nu,
\]
where $\beta_{0}^{\prime}=\beta_{0}+\beta_{z}\mathbb{E}(Z)$ and $\nu=\beta
_{z}(Z^{\prime}-\mathbb{E}(Z))$, a zero mean random variable independent from
$(X,Y)$ that does not contribute in any way to the prediction of $Y$. A better
alternative would be to use the reduced version of $f$ given just by
$\beta_{0}^{\prime}+X^{T}\beta_{x}$, that is equivalent to replacing $Z$ by
$\mathbb{E}(Z)$ in $f(X,Z)$.
\end{itemize}

A more useful replacement for $Z$ is $\mathbb{E}(Z|X)$, the ghost variable for
$Z$, that is the best prediction of $Z$ as a function of $X$ according to
quadratic loss. Note that if there is dependency between $X$ and $Z$,
$|Z-\mathbb{E}(Z|X)|$ is expected to be lower than $|Z-\mathbb{E}(Z)|$, so
$f(X,\mathbb{E}(Z|X))$ is expected to be closer to $f(X,Z)$ than
$f(X,\mathbb{E}(Z))$.
Therefore, when $Z$ is not available, replacing it by $\mathbb{E}(Z|X)$ allows
$X$ to contribute a little bit more in the prediction of $Y$ than replacing
$Z$ by $\mathbb{E}(Z)$. The larger is this extra contribution of $X $, the
smaller is the relevance of $Z$ in the prediction of $Y$, that could be
measured by $\mathbb{E}(L(f(X,Z),f(X,\mathbb{E}(Z|X))))$.
Observe that replacing $Z$ by $\mathbb{E}(Z|X)$ is equivalent to consider the
reduced version $f_{p}(X)$ of $f(X,Z)$, where the function $f_{p}%
(x)=f(x,\mathbb{E}(Z|X=x))$. It should be noted that $f(X,\mathbb{E}%
(Z|X))=\mathbb{E}(f(X,Z)|X)$ for linear functions $f(X,Z)$. It follows that
removing $Z$ or replacing it by $\mathbb{E}(Z|X)$ are equivalent in multiple
linear regression.

Under quadratic loss, if $f$ is additive in $X$ and $Z$
\[
\mathbb{E}(L(f(X,Z),f(X,\mathbb{E}(Z|X))))=\mathbb{E}((s_{2}(Z)-s_{2}%
(\mathbb{E}(Z|X)))^{2}).
\]
If, additionally, $s_{2}(Z)=Z\beta_{z}$,
\[
\mathbb{E}(L(f(X,Z),f(X,\mathbb{E}(Z|X))))=\beta_{z}^{2}\mathbb{E}%
((Z-\mathbb{E}(Z|X))^{2})=
\beta_{z}^{2}\mathbb{E}(\mbox{Var}(Z|X))=\mathbb{E}(\mbox{Var}(s_{2}(Z)|X)).
\]
Observe that $\mathbb{E}(\mbox{Var}(s_{2}(Z)|X))$ and $\beta_{z}^{2}%
\mathbb{E}(\mbox{Var}(Z|X))$ coincide, respectively, with $\mbox{Var}(s_{2}%
(Z))$ and $\beta_{z}^{2}\mbox{Var}(Z)$ when $X$ and $Z$ are independent, but
otherwise the first two would be preferred to the second ones as relevance
measures of $Z$.

The following Section \ref{sec:VarRelSample}
deals with the practical implementation of the three approaches introduced so
far for measuring the relevance of a variable: suppressing it, replacing it by
an independent copy, or by a ghost variable, respectively. Particular
attention is paid to the additive and linear prediction models, and quadratic
loss is assumed.
Let $m(x,z)=\mathbb{E}(Y|X=x,Z=z)$ be the \emph{regression function} of $Y$ on $(X,Z)$,
the best prediction function for $Y$ under quadratic loss.
Any prediction function of $Y$, $f(x,z)$, is also an estimator $\hat{m}(x,z)$ of the
regression function $m(x,z)$, and vice versa.
So from now on we will talk indistinctly of prediction functions or regression function estimators.

\section{Relevance in data sets}

\label{sec:VarRelSample}

Consider a training sample of $n_{1}$ independent realizations of $(X,Z,Y)$,
$\mathcal{S}_{1}=\left\{  (x_{1.i},z_{1.i},y_{1.i}),\right.$
$\left.i=1,\ldots,n_{1}\right\}$. 
Let $(\mathbf{X}_{1},\mathbf{z}_{1},\mathbf{y}_{1})$, with $\mathbf{X}%
_{1}\in\mathbb{R}^{n_{1}\times p}$, $\mathbf{z}_{1}\in\mathbb{R}^{n_{1}%
\times1}$ and $\mathbf{y}_{1}\in\mathbb{R}^{n_{1}\times1}$, be the matrix
representation of the training sample, that is used to estimate the regression
function $m(x,z)$ by $\hat{m}(x,z)$.
The expected quadratic loss of $\hat{m}$ (\emph{Mean Square Prediction Error}) is
\[
\text{MSPE}(\hat{m})=\mathbb{E}\left(  (Y-\hat{m}(X,Z))^{2}\right)  ,
\]
where $(X,Z,Y)$ is independent from the training sample. In order to estimate
$\text{MSPE}(\hat{m})$, a test sample is observed: $n_{2}$ independent
realizations of $(X,Z,Y)$, 
$\mathcal{S}_{2}=\left\{  (x_{2.i},z_{2.i},y_{2.i}),\right.$ $\left.i=1,\ldots,n_{2}\right\}$, 
that are also independent from the
training sample. Let $(\mathbf{X}_{2},\mathbf{z}_{2},\mathbf{y}_{2})$, be the
test sample in matrix format, with $\mathbf{X}_{2}\in\mathbb{R}^{n_{2}\times
p}$, $\mathbf{z}_{2}\in\mathbb{R}^{n_{2}\times1}$ and $\mathbf{y}_{2}%
\in\mathbb{R}^{n_{2}\times1}$, and let $\mathbf{z}_{2}^{\prime}\in
\mathbb{R}^{n_{2}\times1}$ be a random permutation of the elements of the
column vector $\mathbf{z}_{2},$ with $i$-th element $z_{2.i}^{\prime}$.
%
The estimation of $\text{MSPE}(\hat{m})$ from the test sample is
\[
\widehat{\text{MSPE}}(\hat{m})=\frac{1}{n_{2}}\sum_{i=1}^{n_{2}}(y_{2.i}%
-\hat{m}(x_{2.i},z_{2.i}))^{2}.
\]
This is, under mild conditions, a consistent and unbiased estimator of
$\text{MSPE}(\hat{m})$.

In a similar way, the test sample is used to obtain the sampling versions of
the three variable relevance measures introduced in Section
\ref{sec:VarRelPopul}, by
\[
\mbox{Rel}_{V}(Z)=\frac{1}{n_{2}}\sum_{i=1}^{n_{2}}(\hat{m}(x_{2.i}%
,z_{2.i})-\hat{m}_{V}(i))^{2},
\]
where for $V=Om$, relevance for omission, $\hat{m}_{V}(i)=\hat{m}_{p}%
(x_{2.i}),$ for $V=RP,$ relevance for random permutation, $\hat{m}_{V}%
(i)=\hat{m}(x_{2.i},z_{2.i}^{\prime})$ and for $V=Gh$, 
relevance for ghost variable, $\hat{m}_{V}(i)=$ $\hat{m}(x_{2.i},\widehat{z}_{2.i})$ where
$\widehat{z}_{2.i}=\hat{\mathbb{E}}(Z|X=x_{2.i}).$ Scaled versions of these
relevance measures can be defined dividing them by an estimation of the
residual variance ($\widehat{\text{MSPE}}(\hat{m})$, for instance), as in
\citeN{Breiman:2001}. We present the results for the unscaled forms to
simplify the presentation.

Now we proceed to explore the relevance measures when they are computed in the
multiple linear regression model (and, in some cases, in the additive model).

\subsection{Relevance by omission in linear regression}

\label{sec:Rel.Om.lin.mod}

Suppose zero mean variables and the linear regression 
$m(x,z)=x^{T}\beta_{x}+z\beta_{z}$, for which the natural reduced version is $m_{p}(x)=x^{T}\beta_{0}$. 
Let $(\hat{\beta}_{x}%
,\hat{\beta}_{z})$ and $\hat{\beta}_{0}$ be, respectively, the ordinary least
squares (OLS) estimated coefficients of both models, and $\hat{\mathbf{y}%
}_{1.X.z}=\mathbf{X}_{1}\hat{\beta}_{x}+\mathbf{z}_{1}\hat{\beta}_{z}$,
$\hat{\mathbf{y}}_{1.X}=\mathbf{X}_{1}\hat{\beta}_{0}$. Let $\mbox{se}(\hat
{\beta}_{z})$ be the estimated standard error of $\hat{\beta}_{z}$ and let
$t_{z}=\hat{\beta}_{z}/\mbox{se}(\hat{\beta}_{z})$ the standard $t$-test
statistic for the null hypothesis $H_{0}:\beta_{z}=0$. Let $F_{z}=t_{z}^{2}$
be the $F$-test statistic for the same null hypothesis.
Standard computations in the linear model 
(see, e.g., \citeNP{seber2003linear}, 
or check the online supplement, where we have included them 
for the sake of completeness) 
lead to establish the identity
\[
\frac{n_{1}}{\hat{\sigma}^{2}}
\mbox{Rel}_{\mbox{\scriptsize Om}}^{\mbox{\scriptsize Train}}(Z)=F_{z},
\]
where 
\[
\mbox{\rm Rel}_{\mbox{\scriptsize\rm Om}}^{\mbox{\scriptsize\rm Train}}%
(Z)=\frac{1}{n_{1}}(\hat{\mathbf{y}}_{1.X.z}-\hat{\mathbf{y}}_{1.X})^{T}%
(\hat{\mathbf{y}}_{1.X.z}-\hat{\mathbf{y}}_{1.X})
\]
is the relevance by omission of the variable $Z$, evaluated in the training
sample, and 
$\hat{\sigma}^{2}=(\mathbf{y}_{1}-\hat{\mathbf{y}}_{1.X.z})^{T}
(\mathbf{y}_{1}-\hat{\mathbf{y}}_{1.X.z})/(n_{1}-p-1)$.
That is, measuring the relevance of a variable by its omission in the training sample is
equivalent to the standard practice of testing its significance.

Moreover, the same computations indicate that 
$\mbox{Rel}_{\mbox{\scriptsize Om}}^{\mbox{\scriptsize Train}}(Z)=
\hat{\beta}_{z}^{2}\hat{\sigma}_{z.x,n_{1}}^{2}$,
where $\hat{\sigma}_{z.x,n_{1}}^{2}$ is a consistent estimator of
$\sigma_{z.x}^{2}$, the residual variance in the model $Z=X^{T}\alpha
+\varepsilon_{z}$, computed from the training sample.
This is a sampling version of the expression $\beta_{z}^{2}\mathbb{E}(\mbox{Var}(Z|X))$ that we obtained in Section \ref{sec:VarRelPopul} as the value for the relevance by omission or by
ghost variables at population level (we saw there that both coincide for linear models).

When using the test sample to compute the relevance by omission we obtain similar results.
The vectors of predicted values in the test sample are 
$\hat{\mathbf{y}}_{2.X.z}=\mathbf{X}_{2}\hat{\beta}_{x}+\mathbf{z}_{2} \hat{\beta}_{z}$ and
$\hat{\mathbf{y}}_{2.X}=\mathbf{X}_{2}\hat{\beta}_{0}$. Therefore, the
\emph{relevance by omission} of the variable $Z$ is
\[
\mbox{Rel}_{\mbox{\scriptsize Om}}(Z)=\frac{1}{n_{2}}(\hat{\mathbf{y}}%
_{2.X.z}-\hat{\mathbf{y}}_{2.X})^{T}(\hat{\mathbf{y}}_{2.X.z}-\hat{\mathbf{y}%
}_{2.X}).
\]
Then it can be proved (detailed computation can be found in the online supplement) that 
\[
\frac{n_{1}}{\hat{\sigma}^{2}}\mbox{\rm Rel}_{\mbox{\scriptsize\rm Om}}%
(Z)=F_{z} \, \frac{\hat{\sigma}^{2}_{z.x,n_{1},n_{2}}}{\hat{\sigma}%
^{2}_{z.x,n_{1}}}= F_{z} \left(  1 + O_{p}\left(  \min\{n_{1},n_{2}%
\}^{-1/2}\right)  \right)  ,
\]
and
\[
\mbox{Rel}_{\mbox{\scriptsize Om}}(Z)= \hat{\beta}_{z}^{2} \hat{\sigma}%
^{2}_{z.x,n_{1},n_{2}},
\]
where $\hat{\sigma}^{2}_{z.x,n_{1},n_{2}}$ and $\hat{\sigma}^{2}_{z.x,n_{1}}$
are consistent estimators of the same parameter $\sigma^{2}_{z.x}$, the
residual variance in the linear regression model $Z=X^{T}\alpha+\varepsilon
_{z}$.
It follows that, for large values of $n_{1}$ and $n_{2}$, $\hat{\sigma
}_{z.x,n_{1},n_{2}}^{2}/\hat{\sigma}_{z.x,n_{1}}^{2}\approx 1$, and then
\[
\frac{n_{1}}{\hat{\sigma}^{2}}\mbox{Rel}_{\mbox{\scriptsize Om}}(Z)\approx
F_{z},
\]
approximately the same relationship we have found when computing the relevance
by omission in the training sample. 


\subsection{\normalsize Relevance by random permutations in the linear and additive models}

\label{sec:random_copy_add_mod}


The use of random permutations implies independence 
(conditional to the observed test sample)
of the replacing variable
$z_{2,i}^{\prime}$ and the other explanatory variables $x_{2,i}$. We will
analyze directly the additive model and show the results for linear regression
as a particular case. Assume that an \emph{additive model} is fitted in the
training sample%
\[
\hat{m}(x,z)=\hat{\beta}_{0}+\sum_{j=1}^{p}\hat{s}_{j}(x_{j})+\hat{s}%
_{p+1}(z),
\]
$\hat{\beta}_{0}=\sum_{i=1}^{n_{1}}y_{1.i}/n_{1}$, 
$\sum_{i=1}^{n_{1}}\hat{s}_{j}(x_{1.i})/n_{1}=0,\,j=1,\ldots,p$, and
$\sum_{i=1}^{n_{1}}\hat{s}_{p+1}(z_{1.i})/n_{1}=0$ for identifiability
reasons. These identities are only approximately true when taking averages at
the test sample.
Observe that
\[
\mbox{Rel}_{\mbox{\scriptsize RP}}(Z)=\frac{1}{n_{2}}\sum_{i=1}^{n_{2}}%
(\hat{s}_{p+1}(z_{2.i})-\hat{s}_{p+1}(z_{2.i}^{\prime}))^{2}=
\]%
\[
2\frac{1}{n_{2}}\sum_{i=1}^{n_{2}}\hat{s}_{p+1}(z_{2.i})^{2}-2\frac{1}{n_{2}%
}\sum_{i=1}^{n_{2}}\hat{s}_{p+1}(z_{2.i})\hat{s}_{p+1}(z_{2.i}^{\prime
})\approx
\]%
\[
2\frac{1}{n_{2}}\sum_{i=1}^{n_{2}}\hat{s}_{p+1}(z_{2.i})^{2}%
=2\widehat{\mbox{Var}}(\hat{s}_{p+1}(Z)).
\]
The approximation follows from the fact that $\sum_{i=1}^{n_{2}}\hat{s}%
_{p+1}(z_{2.i})\hat{s}_{p+1}(z_{2.i}^{\prime})/n_{2}$ has expected value over
random permutations equal to $\{\sum_{i=1}^{n_{2}}\hat{s}_{p+1}(z_{2.i}%
)/n_{2}\}^{2}\approx0$ and variance of order $O(1/n_{2})$.

In the special case of linear regression, $\hat{s}_{j}(x_{j})=\hat{\beta}%
_{j}(x_{j}-\bar{x}_{1.j})$ for all $j=1,\ldots,p$, and $\hat{s}_{p+1}%
(z)=\hat{\beta}_{z}(z-\bar{z}_{1})$, where $\bar{x}_{1.j}=\sum_{i=1}^{n_{1}%
}x_{1.ij}/n_{1}$ and $\bar{z}_{1}=\sum_{i=1}^{n_{1}}z_{1.i}/n_{1}$. Then,
in this case
\[
\mbox{Rel}_{\mbox{\scriptsize RP}}(Z)\approx2\hat{\beta}_{z}^{2}\frac{1}%
{n_{2}}\sum_{i=1}^{n_{2}}(z_{2.i}-\bar{z}_{1})^{2}\approx2\hat{\beta}_{z}%
^{2}\frac{1}{n_{2}}\sum_{i=1}^{n_{2}}(z_{2.i}-\bar{z}_{2})^{2}=2\hat{\beta
}_{z}^{2}\widehat{\mbox{Var}}(Z).
\]
Observe that $\mbox{Rel}_{\mbox{\scriptsize RP}}(Z)$ is not a multiple of the
statistic $F_{z}$ used to test the null hypothesis $H_{0}:\beta_{z}=0$,
because the variance of $\hat{\beta}_{z}$, the OLS estimator of $\beta_{z}$,
is not a multiple of $1/\widehat{\mbox{Var}}(Z)$, except in the particular
case that $\mathbf{X}_{1}$ and $\mathbf{z}_{1}$ are uncorrelated.


\subsection{Relevance by ghost variables in linear regression}

\label{sec:Rel.Gh.lin.mod}

To get a model agnostic proposal, the training sample $\mathcal{S}_{1}$ should
be used only through the estimated prediction function. Therefore we propose
to estimate $\mathbb{E}(Z|X)$ using the data in the test sample $\mathcal{S}%
_{2}.$ Suppose we fit a simple prediction model (for instance, a multiple
linear regression or an additive model) $\hat{z}_{2.X.i}=\hat{m}_{z}(x_{2.i}%
)$, $i=1,\ldots,n_{2}.$ Let us assume that the regression function is linear
and that it is estimated by OLS in the training sample $(\mathbf{X}%
_{1},\mathbf{z}_{1},\mathbf{y}_{1})\in\mathbb{R}^{n_{1}\times(p+2)} $. Let
$\hat{\beta}_{x}$ and $\hat{\beta}_{z}$ be the estimated coefficients, and let
$\hat{\sigma}^{2}$ be the estimated residual variance. Let $(\mathbf{X}%
_{2},\mathbf{z}_{2},\mathbf{y}_{2})\in\mathbb{R}^{n_{2}\times(p+2)}$ be the
test sample in matrix format. 
We fit by OLS in the test sample to obtain the ghost values for $\bz_2$:
\[
\hat{\mathbf{z}}_{2.2}=\mathbf{X}_{2}\hat{\alpha}_{2},
\]
with $\hat{\alpha}_{2}=(\mathbf{X}_{2}^{T}\mathbf{X}_{2})^{-1}\mathbf{X}%
_{2}^{T}\mathbf{z}_{2}$. Then,
\[
\hat{\mathbf{y}}_{2.X.z}=\mathbf{X}_{2}\hat{\beta}_{x}+\mathbf{z}_{2}%
\hat{\beta}_{z},
\]
when using $(X,Z)$ as predictors, and
\[
\hat{\mathbf{y}}_{2.X.\hat{z}}=\mathbf{X}_{2}\hat{\beta}_{x}+\hat{\mathbf{z}%
}_{2.2}\hat{\beta}_{z},
\]
when using $(X,\hat{Z}_{X})$, that is replacing $Z$ by the ghost variable.
Therefore, the \emph{relevance by a ghost variable} of the variable $Z$ is
\[
\mbox{Rel}_{\mbox{\scriptsize Gh}}(Z)=\frac{1}{n_{2}}(\hat{\mathbf{y}}%
_{2.X.z}-\hat{\mathbf{y}}_{2.X.\hat{z}})^{T}(\hat{\mathbf{y}}_{2.X.z}%
-\hat{\mathbf{y}}_{2.X.\hat{z}}).
\]
The following result connects this relevance measure with the $F$-test for
$\beta_{z}$ The proof can be found in Appendix \ref{Ap:Proof_Rel_Gh}.

\begin{theorem}
\label{th:Rel_Gh} Assume that the regression function of $Y$ over $(X,Z)$ is
linear, that it is estimated by OLS, and that the ghost variable for $Z$ is
also estimated by OLS. Then
\[
\frac{n_{1}}{\hat{\sigma}^{2}}\mbox{\rm Rel}_{\mbox{\scriptsize\rm Gh}}(Z)=
F_{z} \, \frac{\hat{\sigma}^{2}_{z.x,n_{2}}}{\hat{\sigma}^{2}_{z.x,n_{1}}} =
F_{z} \left(  1 + O_{p}\left(  \min\{n_{1},n_{2}\}^{-1/2}\right)  \right)  ,
\]
and
\[
\mbox{\rm Rel}_{\mbox{\scriptsize\rm Gh}}(Z)= \hat{\beta}_{z}^{2} \hat{\sigma
}^{2}_{z.x,n_{2}},
\]
where $\hat{\sigma}^{2}_{z.x,n_{2}}$ and $\hat{\sigma}^{2}_{z.x,n_{1}}$ are
consistent estimators of the same parameter $\sigma^{2}_{z.x}$ (the residual
variance in the linear regression model $Z=X^{T}\alpha+\varepsilon_{z}$), the
first one depending on the test sample, and the second one 
on the training sample.
\end{theorem}


The parallelism between deleting the variable $Z$ and replacing it by a ghost
variable goes farther in the linear regression model. In the linear model
$m_{z}(x)=x^{T}\alpha$, let $\hat{\alpha}_{1}$ and $\hat{\alpha}_{2}$
be the OLS estimators of $\alpha$ in the training and test samples, respectively. 
They are expected to be close each other, because both are OLS estimators of the same
parameter.
This expected proximity and the results stated in 
Appendix \ref{Ap:updating_formula}, lead us to write
\[
\hat{\mathbf{y}}_{2.X.\hat{z}} = \mathbf{X}_{2} \hat{\beta}_{x} +
\hat{\mathbf{z}}_{2.2} \hat{\beta}_{z} = \mathbf{X}_{2} \hat{\beta}_{x} +
\mathbf{X}_{2} \hat{\alpha}_{2} \hat{\beta}_{z} = \mathbf{X}_{2} \left(
\hat{\beta}_{x} + \hat{\alpha}_{2} \hat{\beta}_{z} \right)  \approx
\]
\[
\mathbf{X}_{2} \left(  \hat{\beta}_{x} + \hat{\alpha}_{1} \hat{\beta}_{z}
\right)  = \mathbf{X}_{2}\hat{\beta}_{0} = \hat{\mathbf{y}}_{2.X}.
\]
That is, using the ghost variable $\hat{\mathbf{z}}_{2.2}$ leads to similar
predictions of $Y$ in the test sample than removing the variable $\mathbf{z}
_{1}$ when fitting the model in the training sample. But using the ghost
variable has two clear advantages: first, only one model has to be fitted in
the training sample (the model with all the explanatory variables), and
second, the estimated prediction function is the only element we have to save
from the training sample (and, consequently, our proposal is model agnostic).
As a consequence, from now on we will no longer consider the relevance by
omission separately.

As a last remark, we would like to emphasize that the relevance by a ghost
variable allows us to approximate a very relevant statistic in the linear
regression model, namely the $F$-statistic for testing the null hypothesis
$H_{0}:\beta_{z}=0$. This approximation only requires the estimated prediction
function from a training sample and a test sample. Therefore the relevance measure
by a ghost variable allows us to evaluate a pseudo $F$-statistic for any
statistical or algorithmic prediction model.

\section{\large Understanding sets of variables: The relevance matrix}

\label{sec:Rel.Matrix}

Variable relevance measures introduced so far have certain parallels with the
usual practice in the detection of influential cases and outliers in
regression, where each data is removed in turns and the effect of these
deletions on the values predicted by the model (or on the estimates of the
parameters) are computed. \citeN{PenyaYohai:1995} introduced the
\emph{influence matrix} by first looking at the influence vectors that measure
the effect of deleting each observation on the vector of forecasts for the
whole data set and then computing the $n\times n$ symmetric matrix that has
the scalar products between these vectors. Thus the influence matrix has in
the diagonals Cook's statistics and outside the diagonal the scalar products
of these effects. These authors showed that eigen-structure of the influence
matrix contains useful information to detect influential subsets or multiple
outliers avoiding masking effects. We propose a similar idea by computing a
\emph{case-variable} relevance matrix (denoted by $\mathbf{A} $ below). We
assume first that the variables will be substituted by their ghost variables as
explained before, and afterwards we will consider the substitution by random permutations.

In this section we deal with all the explanatory variables in a symmetric way
and consider the prediction of $Y$ from the $p$ components of $X$ through the
regression function $m(x)= \mathbb{E} (Y|X=x)$. It is assumed that a training
sample $( \mathbf{X} _{1}, \mathbf{y} _{1})\in\mathbb{R} ^{n_{1}\times(p+1)}$
has been used to build and estimator $\hat{m}(x)$ of $m(x)$, and that an
additional test sample $( \mathbf{X} _{2}, \mathbf{y} _{2})\in\mathbb{R}
^{n_{2}\times(p+1)}$ is available.

Let $\mathbf{x}_{2.1},\ldots,\mathbf{x}_{2.p}$ be the columns of
$\mathbf{X}_{2}$. For $j=1,\ldots,p$, let $\mathbf{X}_{2.[j]}=(\mathbf{x}%
_{2.1},\ldots,\mathbf{x}_{2.(j-1)},$ $\mathbf{x}_{2.(j+1)},\ldots,\mathbf{x}%
_{2.p})$ be the matrix $\mathbf{X}_{2}$ without the $j$-th column, and
$\mathbf{H}_{2.[j]}=\mathbf{X}_{2.[j]}(\mathbf{X}_{2.[j]}\mathbf{X}%
_{2.[j]})^{-1}\mathbf{X}_{2.[j]}^{\mbox{\scriptsize T}}$ be the projection
matrix on the column space of $\mathbf{X}_{2.[j]}$. Let $\hat{\mathbf{x}%
}_{2.j}=\mathbf{H}_{2.[j]}\mathbf{x}_{2.j}$ be projection of $\mathbf{x}%
_{2.j}$ over the column space of the other columns of $\mathbf{X}_{2}$. Thus
$\hat{\mathbf{x}}_{2.j}$ is the $j$-th ghost variable. Note that alternative
regression models (additive models, for instance) could be also used to define
the ghost variable. Let
\[
\mathbf{X}_{2.\hat{\jmath}}=(\mathbf{x}_{2.1},\ldots,\mathbf{x}_{2.j-1}%
,\hat{\mathbf{x}}_{2.j},\mathbf{x}_{2.j+1},\ldots,x_{2.p})
\]
be the regressor matrix in the test sample where the $j$-th variable has been
replaced by the $j$-th ghost variable. We use $\hat{m}%
(\mathbf{X}_{2})$ to denote the $n_{2}$-dimensional column vector with $i$-th
element equal to the result of applying the function $\hat{m}$ to the $i$-th
row of $\mathbf{X}_{2}$. Let
\[
\hat{\mathbf{Y}}_{2}=\hat{m}(\mathbf{X}_{2})\mbox{ and }\hat{\mathbf{Y}%
}_{2.\hat{\jmath}}=\hat{m}(\mathbf{X}_{2.\hat{\jmath}}).
\]
We define the $n_{2}\times p$ matrix
\[
\mathbf{A}=(\hat{\mathbf{Y}}_{2}-\hat{\mathbf{Y}}_{2.\hat{1}},\ldots
,\hat{\mathbf{Y}}_{2}-\hat{\mathbf{Y}}_{2.\hat{p}}).
\]
Observe that the element $a_{ij}$ of $\mathbf{A}$, $i=1,\ldots,n_{2}$,
$j=1,\ldots,p$, measures the change in the response prediction for the $i$-th
case in the test sample, when the $j$-th variable has been replaced by its
ghost variable. Finally, we define the \emph{relevance matrix} as the $p\times
p$ matrix
\[
\mathbf{V}=\frac{1}{n_{2}}\mathbf{A}^{\mbox{\scriptsize T}}\mathbf{A}.
\]
Then, the element $(j,k)$ of $\mathbf{V}$ is
\[
v_{jk}=\frac{1}{n_{2}}\sum_{i=1}^{n_{2}}(\hat{m}(x_{2.j.i})-\hat{m}%
(x_{2.\hat{\jmath}.i}))(\hat{m}(x_{2.k.i})-\hat{m}(x_{2.\hat{k}.i})),
\]
where $x_{2.j.i}$ and $x_{2.\hat{\jmath}.i}$ are, respectively, the $i$-th
element of $\mathbf{x}_{2.j}$ and $\hat{\mathbf{x}}_{2.j}$. In particular, the
diagonal of the relevance matrix $\mathbf{V}$ has elements
\[
v_{jj}=\mbox{Rel}_{\mbox{\scriptsize Gh}}(X_{j}),\,j=1,\ldots,p.
\]
The advantage of working with the matrix $\mathbf{V}$, instead of just
computing univariate relevance measures, is that $\mathbf{V}$ contains
additional information in its out-of-the-diagonal elements, that we are
exploring through the examination of its eigen-structure.
If, for $j=1\ldots,p$,
$(1/n_{2})\sum_{i}\hat{m}(x_{2.j.i})=(1/n_{2})\sum_{i}\hat
{m}(x_{2.\hat{\jmath}.i})$ then
$\mathbf{V}$ is the variance-covariance matrix of $\mathbf{A}$.

\subsection{The relevance matrix for linear regression}

\label{sec:Relev.Matrix.lin.mod} Consider the particular case that $\hat
{m}(x)$ is the OLS estimator of a multiple linear regression. Then
\[
\hat{m}(x)=x^{\mbox{\scriptsize T}}\hat{\beta},\mbox{ with }\hat{\beta
}=(\mathbf{X}_{1}^{\mbox{\scriptsize T}}\mathbf{X}_{1})^{-1}\mathbf{X}%
_{1}^{\mbox{\scriptsize T}}\mathbf{Y}_{1}.
\]
Therefore, the vector of predicted values in the test sample is $\hat
{\mathbf{Y}}_{2}=\mathbf{X}_{2}\hat{\beta}$. Writing $\hat{\beta}=(\hat
{\beta_{1}},\ldots,\hat{\beta_{p}})^{\mbox{\scriptsize T}}$, the predicted
values when using the $j$-th ghost variable is
\[
\hat{\mathbf{Y}}_{2.\hat{\jmath}}=\mathbf{X}_{2.\hat{\jmath}}\hat{\beta
}=\mathbf{X}_{2}\hat{\beta}-(\mathbf{x}_{2.j}-\hat{\mathbf{x}}_{2.j}%
)\hat{\beta}_{j}=\hat{\mathbf{Y}}_{2}-(\mathbf{x}_{2.j}-\hat{\mathbf{x}}%
_{2.j})\hat{\beta}_{j},
\]
the matrix $\bA$ is
\[
\mathbf{A}=(\hat{\mathbf{Y}}_{2}-\hat{\mathbf{Y}}_{2.\hat{1}},\ldots
,\hat{\mathbf{Y}}_{2}-\hat{\mathbf{Y}}_{2.\hat{p}})=(\mathbf{X}_{2}%
-\hat{\mathbf{X}}_{2})\mbox{diag}(\hat{\beta}),
\]
where $\hat{\mathbf{X}}_{2}$ is the matrix with each variable replace by its
ghost variable. The relevance matrix is
\begin{equation}
\mathbf{V}=\frac{1}{n_{2}}\mbox{diag}(\hat{\beta})(\mathbf{X}_{2}%
-\hat{\mathbf{X}}_{2})^{\mbox{\scriptsize T}}(\mathbf{X}_{2}-\hat{\mathbf{X}%
}_{2})\mbox{diag}(\hat{\beta})=\mbox{diag}(\hat{\beta})\mathbf{G}%
\mbox{diag}(\hat{\beta}),\label{matrixV}%
\end{equation}
where $\mathbf{G}=(1/n_{2})(\mathbf{X}_{2}-\hat{\mathbf{X}}_{2}%
)^{\mbox{\scriptsize T}}(\mathbf{X}_{2}-\hat{\mathbf{X}}_{2})$. The elements
$(j,k)$ of $\mathbf{G}$ and $\mathbf{V}$ are, respectively,
\[
g_{jk}=\frac{1}{n_{2}}(\mathbf{x}_{2.j}-\hat{\mathbf{x}}_{2.j}%
)^{\mbox{\scriptsize T}}(\mathbf{x}_{2.k}-\hat{\mathbf{x}}_{2.k}%
),\mbox{ and }v_{jk}=\hat{\beta}_{j}\hat{\beta}_{k}g_{jk}.
\]
Observe that, in the regression of $\mathbf{x}_{2.j}$ over $\mathbf{X}%
_{2.[j]}$, $\hat{\sigma}_{[j]}^{2}=g_{jj}$ is the residual variance estimation
that uses $n_{2}$ as denominator.
The following result summarizes the properties of the relevance matrix
$\mathbf{V}$ and, in particular, its relationship with the partial correlation
matrix. The proof can be found in Appendix \ref{Ap:Proof_relevance_matrix}.

\begin{theorem}
\label{th:relevance_matrix} Let $\mathbf{P}$ be the matrix that contains the
partial correlation coefficients in the test sample as non-diagonal elements
and has $-1$ in the diagonal. Then
\[
\mathbf{G} =\frac{1}{n_{2}}(\mathbf{X}_{2}-\hat{\mathbf{X}}_{2}%
)^{\mbox{\scriptsize T}} (\mathbf{X}_{2}-\hat{\mathbf{X}}_{2}) =
-\mbox{diag}(\hat{\sigma}_{[1]},\ldots,\hat{\sigma}_{[p]}) \,\mathbf{P}%
\,\mbox{diag}(\hat{\sigma}_{[1]},\ldots,\hat{\sigma}_{[p]}),
\]
and consequently
\[
\mathbf{V}=-\mbox{diag}(\hat{\beta})\, \mbox{diag}(\hat{\sigma}_{[1]}%
,\ldots,\hat{\sigma}_{[p]})\,\mathbf{P}\,\mbox{diag}(\hat{\sigma}_{[1]}%
,\ldots,\hat{\sigma}_{[p]})\,\mbox{diag}(\hat{\beta}).
\]
Therefore $\mbox{Rel}_{\mbox{\scriptsize Gh}}(X_{j})$, the $j$-th element of
the diagonal of $\mathbf{V} $, admits the alternative expression
\[
\mbox{Rel}_{\mbox{\scriptsize Gh}}(X_{j})=\hat{\beta}_{j}^{2}\hat{\sigma
}_{[j]}^{2},
\]
and the partial correlation coefficient between variables $j$ and $k$ when
controlling by the rest of variables can be computed as
\[
\hat{\rho}_{jk.R}= -\frac{g_{jk}}{\sqrt{g_{jj}g_{kk}}}= -\frac{v_{jk}}%
{\sqrt{v_{jj}v_{kk}}}.
\]

\end{theorem}

The expressions for the partial correlation coefficient appearing in Theorem
\ref{th:relevance_matrix} reminds the well known formula
\[
\hat{\rho}_{jk.R}=-\frac{s^{jk}}{\sqrt{s^{jj}s^{kk}}},
\]
where the $s^{jk}$ is the element $(j,k)$ of $\mathbf{S}_{2}^{-1}$, the
inverse of the covariance matrix of the test sample $\mathbf{X} _{2}$,
$\mathbf{S}_{2}$. This coincidence, and the observation that $s^{jj}$ is the
inverse of $(\mathbf{x}_{2.j}-\hat{\mathbf{x}}_{2.j})^{\mbox{\scriptsize T}}
(\mathbf{x}_{2.j}-\hat{\mathbf{x}}_{2.j})$ (a consequence of the inverse
formula for a block matrix; see Appendix \ref{Ap:updating_formula}), imply the
next Corollary.

\begin{corol}
\label{corol:matrix_G} Let $\mathbf{S}_{2}$ be the covariance matrix of the
test sample $\mathbf{X} _{2}$. $\mathbf{G}$ and $\mathbf{S}_{2}^{-1}$ verify
that
\[
\mathbf{G}=\frac{n_{2}-1}{n_{2}}\, \mbox{diag}(\hat{\sigma}_{[1]}^{2}%
,\ldots,\hat{\sigma}_{[p]}^{2}) \,\mathbf{S}_{2}^{-1}\,\mbox{diag}
(\hat{\sigma}_{[1]}^{2},\ldots,\hat{\sigma}_{[p]}^{2}).
\]

\end{corol}

\subsection{\normalsize Relevance matrix by random permutations in linear regression}

We analyze now the structure of the relevance matrix $\mathbf{V}$
when random permutations are used instead of ghost variables. We focus in the
case of multiple linear regression. 
Define
\[
\tilde{\mathbf{A}}=(\mathbf{X}_{2}-\mathbf{X}_{2}^{\prime})\mbox{diag}(\hat
{\beta}),
\]
where the $j$-th column of matrix $\mathbf{X}_{2}^{\prime}$ is $\mathbf{x}%
_{2.j}^{\prime}$, a random permutation of $\mathbf{x}_{j}$. 
Therefore,
\[
\tilde{\mathbf{V}}=\frac{1}{n_{2}}\tilde{\mathbf{A}}^{\mbox{\scriptsize T}}%
\tilde{\mathbf{A}}=\frac{1}{n_{2}}\mbox{diag}(\hat{\beta})(\mathbf{X}%
_{2}-\mathbf{X}_{2}^{\prime})^{\mbox{\scriptsize T}}(\mathbf{X}_{2}%
-\mathbf{X}_{2}^{\prime})\mbox{diag}(\hat{\beta})\approx2\,\mbox{diag}(\hat
{\beta})\,\mathbf{S}_{2}\,\mbox{diag}(\hat{\beta})=
\]%
\[
2\,\mbox{diag}(\hat{\beta})\,\mbox{diag}(S_{1},\ldots,S_{p})\,\mathbf{R}%
\,\mbox{diag}(S_{1},\ldots,S_{p})\,\mbox{diag}(\hat{\beta}),
\]
where $S_{j}^{2}$ is the sample variance (computed dividing by $n_{2}$) of
$\mathbf{x}_{j}$, and $\mathbf{R}$ is the correlation matrix of the test
sample $\mathbf{X}_{2}$. The \emph{``approximately equal to''} sign ($\approx
$) indicates that $\mathbf{X}_{2}^{\mbox{\scriptsize T}}\mathbf{X}_{2}%
^{\prime}$ is a matrix with elements approximately equal to $0$, because $0$
is their expected value under random permutations. We conclude that the
correlation structure of $\tilde{\mathbf{V}}$ coincides with that of the
sample correlation matrix $\mathbf{R}$, and it has diagonal elements
$2\hat{\beta_{j}}^{2}S_{j}^{2}=\mbox{Rel}_{\mbox{\scriptsize RP}}(X_{j})$. 

We have found analogous expressions for $\mathbf{V} $ and $\tilde{ \mathbf{V}
}$ that allow us to identify the main differences between both relevance
matrices. First, $\mathbf{V} $ is related with partial correlations, while
$\tilde{ \mathbf{V} }$ is associated with standard correlations. Second, the
expression of $\mathbf{V} $ includes the estimated residual variances in the
regressions of each variable over the rest, while the usual sample variances
appear in the expression of $\tilde{ \mathbf{V} }$. These findings suggest
that the eigen-structure of $\tilde{ \mathbf{V} }$ will probably be related
with the principal component analysis of the test sample explanatory matrix
$\mathbf{X} _{2}$, but hopefully new knowledge can be found when analyzing the
eigen-structure of $\mathbf{V} $.

\section{Examples}\label{sec:examples}

We analyze now some examples of the performance of the relevance measures we
have discussed. In Section \ref{sec:SyntData} we use synthetic data, while we
use real data in Section \ref{sec:RealData}.

\subsection{Synthetic data sets}

\label{sec:SyntData}

We introduce three simulated examples to show how relevance 
measures 
work in practice. 
The first two examples have only three and five explanatory variables and are  toy examples where we can look at the details of the procedure in linear regression and additive models. 
The third one is a regression example with 200 explanatory variables. 

\subsubsection*{Example 1: A small linear model}

Our first example 
ollows the lines of the hypothetical situation outlined in the Introduction.
is the linear model $Y = X_{1} + X_{2} + X_{3} + \varepsilon$, where
$X_{1}$ is independent from $X_{2}$ and $X_{3}$, which are highly correlated
($\rho=.95$). Additionally it is assumed that $(X_{1},X_{2},X_{3})$ is
multivariate normal with zero mean, and that $\mbox{Var}(\varepsilon)=\mbox{Var}(X_{i})=1$,
for $i=1,2,3$.

We generate a training sample of size $n_{1}=2000$, and a test sample of size
$n_{2}=1000$ according to this setting. Using the training set, we fit a
linear model ($\hat{y}=\hat\beta_{0}+\hat\beta_{1} x_{1} +\hat\beta_{2} x_{2}
+\hat\beta_{3} x_{3}$), that presents a coefficient of determination
$R^{2}=0.8326$. The $t$-values corresponding to the estimated coefficients
$\hat\beta_{i}$, $i=1,2,3$, are, respectively, $43.549$, $14.514$ and
$13.636$, all of them highly significant. The $t$-value corresponding to
$\hat\beta_{1}$ is the largest one because the variance of $\hat\beta_{2}$ and
$\hat\beta_{3}$ are larger than that of $\hat\beta_{1}$ as $X_{2}$ and $X_{3}$
are strongly correlated.

%
%
%

We compute the variable relevance, as well as the relevance matrix, using
first ghost variables. Figure \ref{fig:Ex5_Relev_GH} summarizes our findings.
The relevance of each variable are represented in the first plot. We can see
that the relevance by ghost variables of $X_{1}$ is much larger than that of
$X_{2}$ and $X_{3}$. The second graphic in the first row compare the values of
the relevance by ghost variables with the corresponding $F$-statistics
(conveniently transformed). It can be seen that for every explanatory variable
both values are almost equal. Blue dashed lines in these two graphics indicate
the critical value beyond which an observed relevance can not be considered
null, at significance level $\alpha=0.01$. According to Theorem
\ref{th:Rel_Gh}, this critical value is computed as
\[
F_{1,n-p-1,1-\alpha} \frac{\hat{\sigma}^{2}}{n_{1}},
\]
where $F_{1,n-p-1,1-\alpha}$ is the $(1-\alpha)$ quantile of a $F_{1,n-p-1}$
distribution ($p=3$ in this example), and $\hat{\sigma}^{2}$ is the estimated
residual variance.

The methods we propose are not limited to reproducing the significance statistics but, through the relevance matrix, provide information on the joint effect that groups of variables have on the response. This goes beyond the standard output even in the linear model.
In the present example, the 
relevance matrix by ghost variables is
\[
\mathbf{V} = \left(
\begin{array}
[c]{rrr}%
0.9259 & 0.0003 & 0.0005\\
0.0003 & 0.1055 & -0.0924\\
0.0005 & -0.0924 & 0.0898
\end{array}
\right)  ,
\]
that essentially is the variance-covariance matrix of the matrix $\mathbf{A}
=(\hat{\mathbf{Y}}_{2}-\hat{\mathbf{Y}}_{2.\hat{1}}, \hat{\mathbf{Y}}_{2}%
-\hat{\mathbf{Y}}_{2.\hat{2}}, \hat{\mathbf{Y}}_{2}-\hat{\mathbf{Y}}%
_{2.\hat{3}})$. Observe that $v_{23}<0$. The negative sign was expected
because of Theorem \ref{th:relevance_matrix}, which states that $v_{23}%
=-\hat{\beta}_{2}\hat{\beta}_{3} \hat{\sigma}_{[2]}\hat{\sigma}_{[3]}\hat
{\rho}_{23.1}$, where $\hat{\rho}_{23.1}$ is the partial correlation
coefficient between the second and third variables, computed in the test
sample. Observe that in this example $\hat{\rho}_{23.1}$ is close to the
correlation coefficient ($0.95$) because $X_{2}$ and $X_{3}$ are independent
from $X_{1}$, and $\hat{\beta}_{j}\approx\beta_{j}=1$, $j=2,3$, so $v_{23}$
must be negative. From $v_{22}$, $v_{33}$ and $v_{23}$, we compute the
correlation between the second and third columns in matrix $\mathbf{A} $ and
obtain a value of $-0.9493$.

The last plot in the upper row of Figure \ref{fig:Ex5_Relev_GH} shows the
eigenvalues of matrix $\mathbf{V} $, and the plots in the lower row represent
the components of each eigenvector. These plots give the same information as
the Principal Component Analysis (PCA) of matrix $\mathbf{A} $: the first PC
coincides with the first column of $\mathbf{A} $, the second PC is a multiple
of the difference of the second and third columns of $\mathbf{A} $ (that have
a strong negative correlation), and the third PC, that has an eigenvalue very
close to $0$, indicates that the sum of the second and third columns of
$\mathbf{A} $ is almost constant.

In terms of variable relevance, from Figure \ref{fig:Ex5_Relev_GH} we take the
following conclusions: (i) the first explanatory variable is the most relevant
and it affects the response variable in a isolated way (there is a unique
eigenvector of $\mathbf{V} $ associated with $X_{1}$); (ii) the other two
variables have similar relevance, and their relevance are strongly related
(there is a pair of eigenvectors associated with $X_{2}$ and $X_{3}$).

\begin{figure}[p]
\begin{center}
\includegraphics[scale=.4]{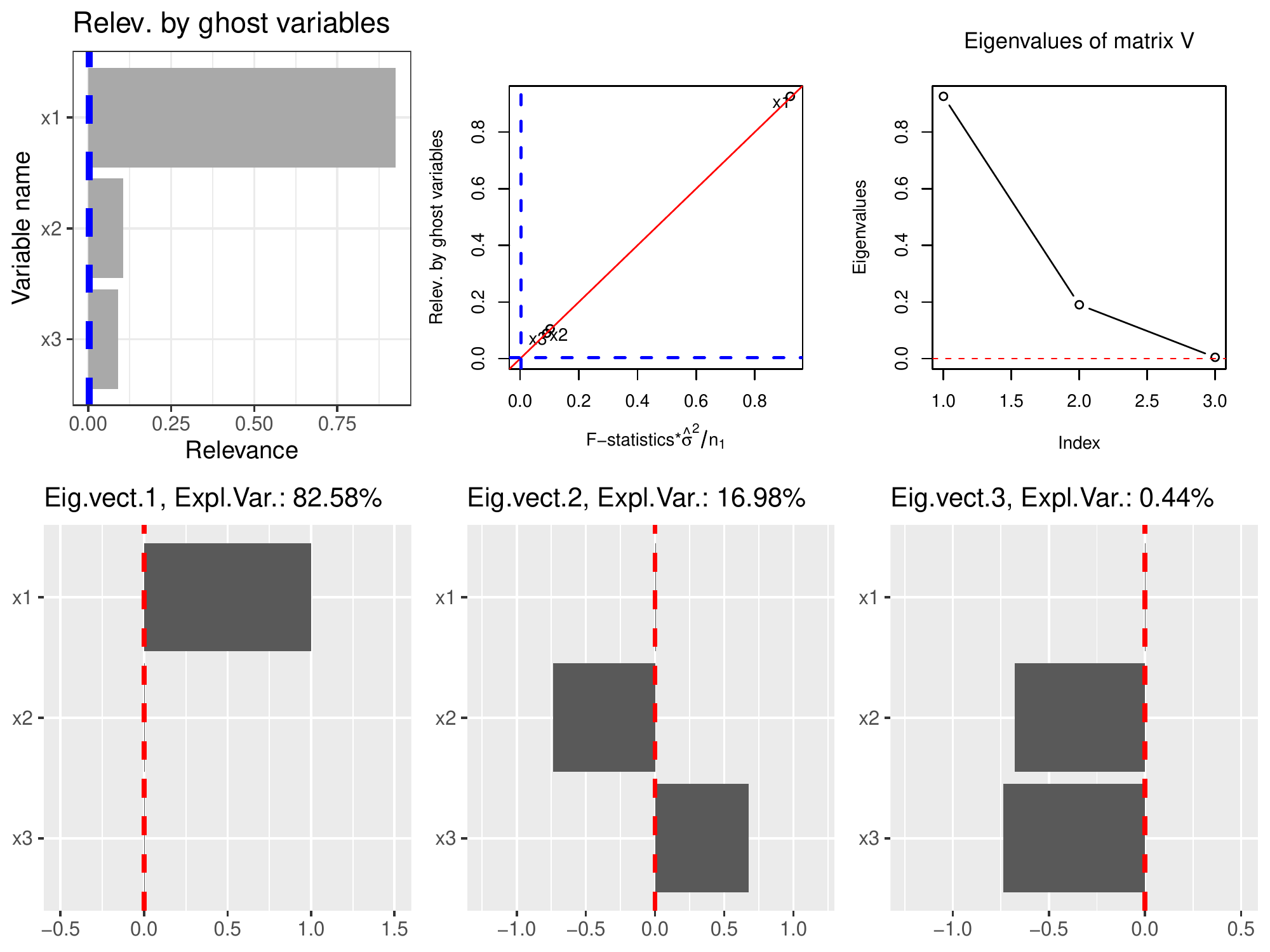}
\end{center}
\caption{Relevance by ghost variables in Example 1. The blue dashed lines mark
the critical values for testing null relevance, $\alpha=0.01$. The dashed red
lines are placed at 0.}%
\label{fig:Ex5_Relev_GH}%
\end{figure}

\begin{figure}[p]
\begin{center}
\includegraphics[scale=.4]{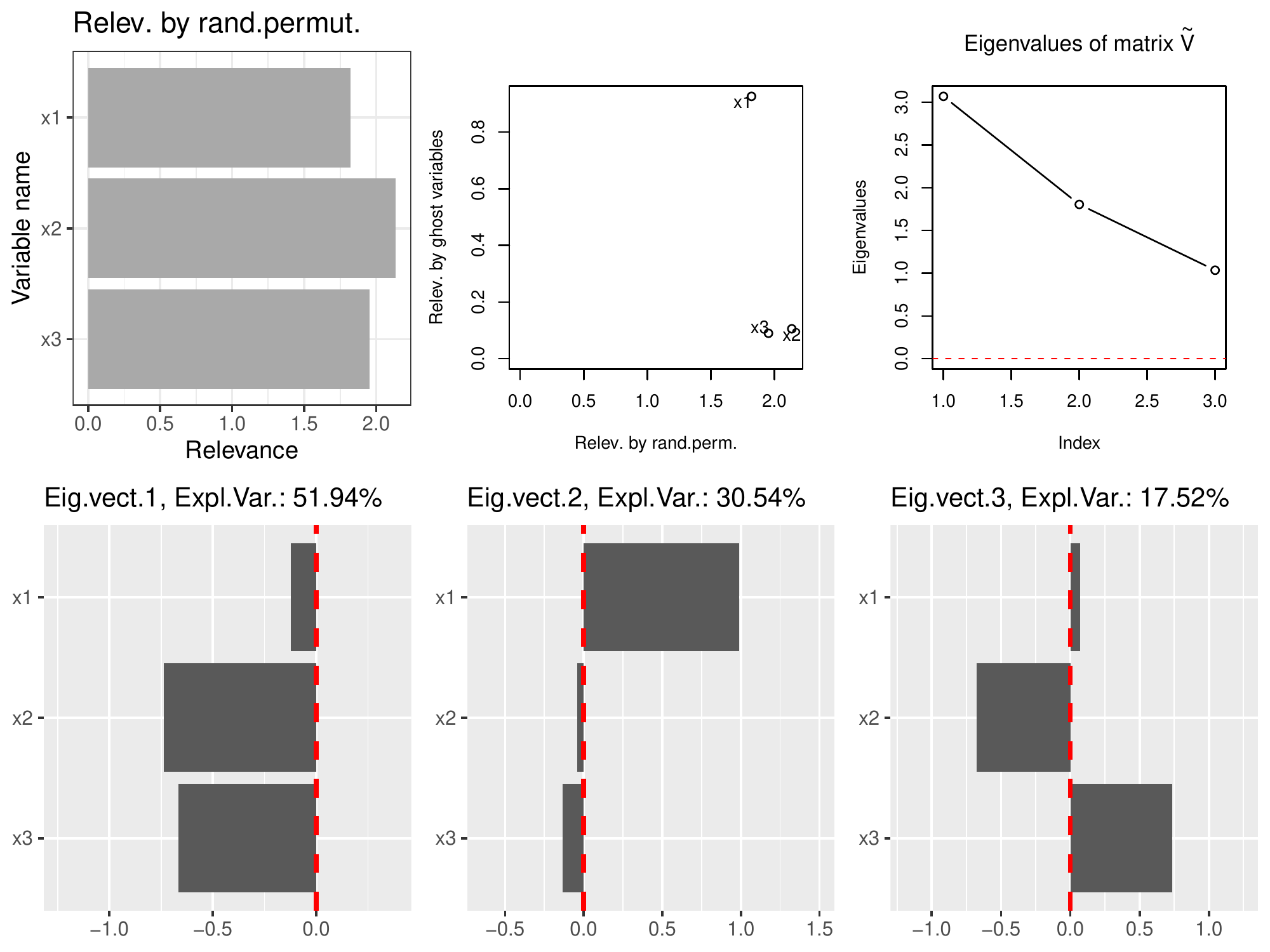}
\end{center}
\caption{Relevance by random permutations in Example 1. The dashed red lines
are placed at 0.}%
\label{fig:Ex5_Relev_RP}%
\end{figure}

Now we study the results obtained when computing the relevance by random
permutations, shown in Figure \ref{fig:Ex5_Relev_RP} (the structure of this
figure is similar to that of Figure \ref{fig:Ex5_Relev_GH}, with the only
exception that now the second plot in the first row shows the relationship
between the relevance by ghost variables and the relevance by random
permutations). We can see that the three variables have a similar relevance,
when computing it by random permutations. In this case the three eigenvalues
are far from $0$. Again there is one eigenvector associated exclusively to
$X_{1}$, and the other two are linear combinations of $X_{2}$ and $X_{3}$. Now
the first eigenvector has coefficients with the same sign in $X_{2}$ and
$X_{3}$, indicating that the second and third columns of $\tilde{ \mathbf{A}
}$ are positively correlated, as it was expected (because the element $(2,3)$
of $\tilde{ \mathbf{V} }$ has the same sign as the correlation coefficient
between $X_{2}$ and $X_{3}$).

The analysis of the relevance by random permutations have shown that (i) the
three explanatory variables have similar relevance, and that (ii) the second
and third variables have a joint influence on the response variable. We
conclude that the information provided by the ghost variables is more useful
than those obtained from random permutations.

\subsubsection*{Example 2: An additive model}
In this example we generate data according to the model
\[
Y = \cos(X_{1}) + \left\{  \frac{1}{2}(\cos(X_{2})+\cos(X_{3})) + \frac{1}{2}
X_{2} X_{3} \right\}  + \cos(X_{4}) + \cos(X_{5}) +\varepsilon,
\]
that is additive in $X_{1}$, $(X_{2},X_{3})$, $X_{4}$ and $X_{5}$. The zero
mean residual is normal with $\mbox{Var}(\varepsilon)=1/4$, and the
explanatory variables $(X_{1},X_{2},X_{3},X_{4},X_{5})$ are multivariate
normal with 0 means an covariance matrix
\[
\left(
\begin{array}
[c]{ccccc}%
1 & .95 & 0 & 0 & 0\\
.95 & 1 & 0 & 0 & 0\\
0 & 0 & 1 & .95 & .95\\
0 & 0 & .95 & 1 & .95\\
0 & 0 & .95 & .95 & 1
\end{array}
\right)  .
\]
So there are two independent blocks of regressors, $(X_{1},X_{2})$ and
$(X_{3},X_{4},X_{5})$, that are linked by the regression function because one
of the additive terms jointly depends on two variables, $X_{2}$ and $X_{3}$,
each belonging to one of the two blocks.

Training and test samples have been generated, with sizes $n_{1}=2000$ and
$n_{2}=1000$, respectively. An additive model has been fitted to the training
set using function \texttt{gam} from the R library \texttt{mgcv}
(\citeNP{Wood:2017:GAM2nd}) using the right model specification:
\verb|y~s(x1)+s(x2,x3)+s(x4)+| \verb|s(x5)|. The fitted model presents an
adjusted coefficient of determination $R^{2}=0.873$.



\begin{figure}[p]
\begin{center}
\includegraphics[scale=.4]{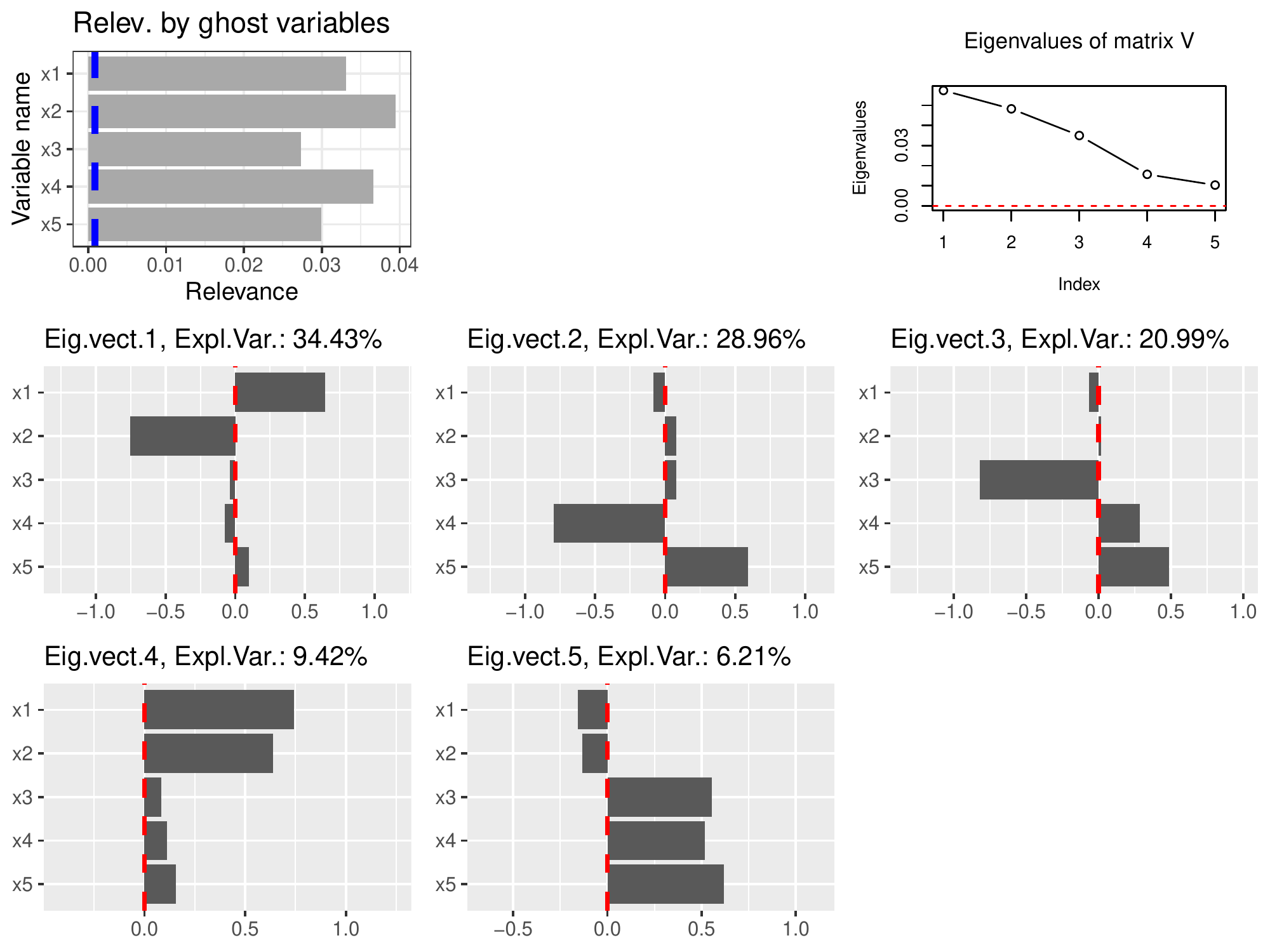}
\end{center}
\caption{Relevance by ghost variables in Example 3.}%
\label{fig:Ex7_gam_Relev_GH}%
\end{figure}

\begin{figure}[p]
\begin{center}
\includegraphics[scale=.4]{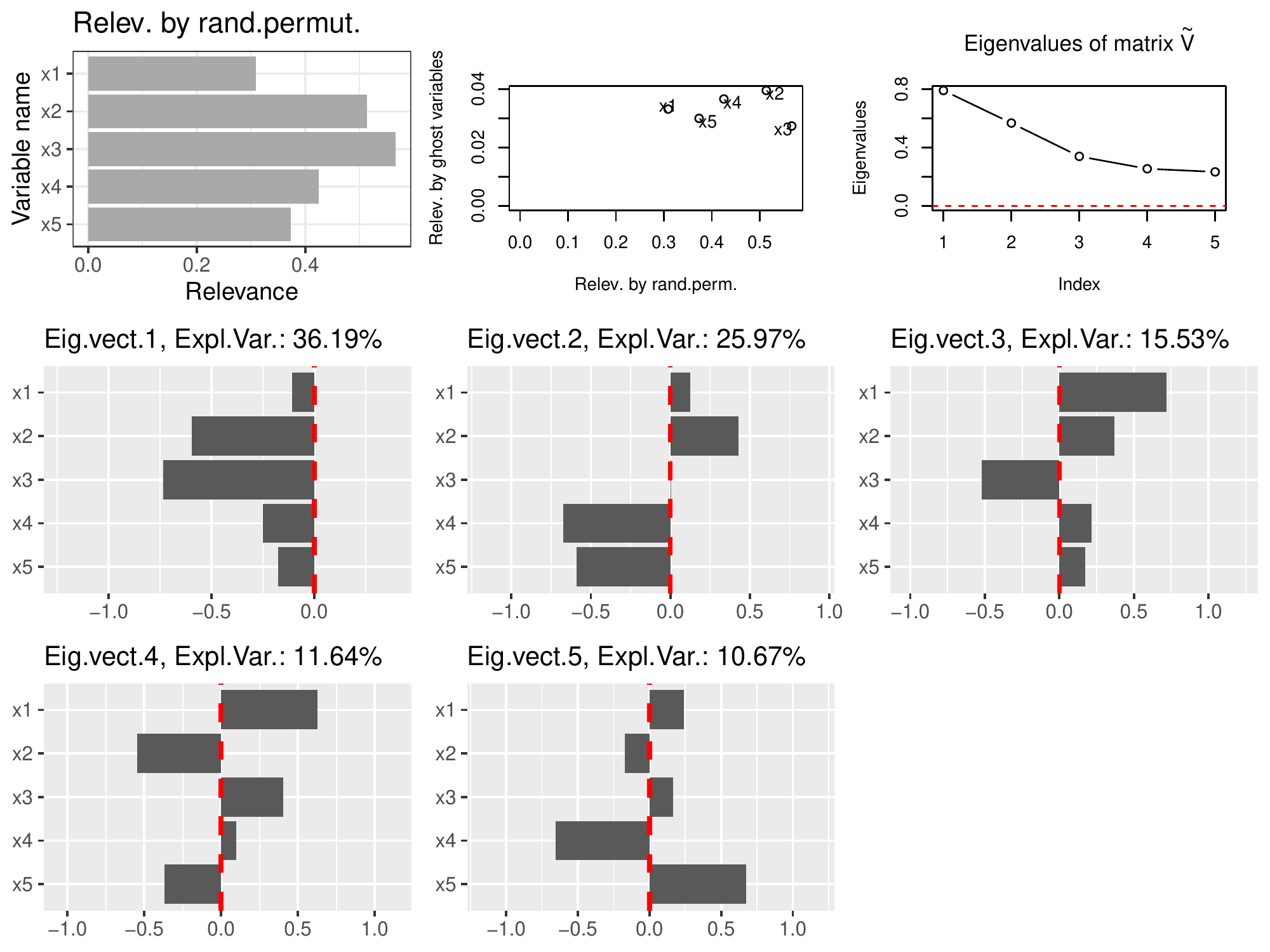}
\end{center}
\caption{Relevance by random permutations in Example 3.}%
\label{fig:Ex7_gam_Relev_RP}%
\end{figure}

Figure \ref{fig:Ex7_gam_Relev_GH} shows the relevance by ghost variables
results. All the explanatory variables have similar relevance. The
eigen-structure of relevance matrix $\mathbf{V} $ reveals the following facts.
The first and fourth eigenvectors are associated mainly with the effects of
$X_{1}$ and $X_{2}$ on the fitted function, while the other three eigenvectors
are associated with $X_{3}$, $X_{4}$ and $X_{5}$. So the presence of two
separate blocks of regressors is clear. Moreover, in the second block, the
role of $X_{3}$ is different from those of $X_{4}$ and $X_{5}$ (see
eigenvectors 2 and 3).

Regarding relevance by random permutations, the results are in Figure 
\ref{fig:Ex7_gam_Relev_RP}. In this case the relevance of $X_{2}$ and $X_{3}$
is larger than that of the other explanatory variables. The eigen-structure of
relevance matrix $\tilde{ \mathbf{V} }$ is different from that of $\mathbf{V}
$. 
The first eigenvector 
indicates 
that variables $X_{2}$ and $X_{3}$ jointly affect the response $Y$ (something already known because the model includes the term {\tt s(x2, x3)}). 
Eigenvectors 2, 3 and 4 point out that it is
difficult to separate the effects of all the regressors on the response.
Finally, the fifth eigenvector is an indication that $X_{4}$ and $X_{5}$ have
a certain differentiated influence on $Y$.

In this example we have seen again that the results from using ghost variables
are more useful than those from random permutations: in Figure
\ref{fig:Ex7_gam_Relev_GH} we discover the two existing separate blocks of
explanatory variables, whereas in Figure \ref{fig:Ex7_gam_Relev_RP} we did not
found anything new as the joint influence over $Y$ of $X_{2}$ and $X_{3}$ was
already included in the fitted model.

\subsubsection*{Example 3: A large linear model}
We consider now a linear model with 200 standard normal explanatory variables, divided in four independent blocks of 50 variables each:
$X_{i.j}$, $i=1,\ldots,4$, $j=1,\ldots,50$.
The first and third blocks include independent variables.
The variables in each of the other two blocks are correlated:
$\mbox{cor}(X_{i.j}, X_{i.k})=0.95$ for $i=2,4$, and $j\ne k$.
The response variable $Y$ is generated as
\[
Y=\frac{1}{2}\sum_{j=1}^{50} X_{1.j} + \sum_{j=1}^{50} X_{2.j} + \varepsilon,
\]
with $\varepsilon \sim N(0,1)$.
A training sample of size $n_{1}=2000$, and a test sample of size
$n_{2}=1000$ are generated according to this model. 
Using the training set, a linear model is fitted.

The variable relevance for the 200 explanatory variables are computed, first using ghost variables and then by random permutations.
Figure \ref{fig:Ex_3_200} summarizes the findings. 
The top left panel indicates that the relevance by ghost variables (that are quite similar to the transformed $F$-statistics) is larger for the 50 variables in the fist block (uncorrelated variables related with $Y$) than for those in the second block (correlated variables related with $Y$), while the rest of variables (blocks 3 and 4, both independent from $Y$) are almost irrelevant.
The relevance by ghost variables takes into account the dependence structure, detects that 
each variable in the first block contains exclusive information about the response (this is not the case in the second block), and concludes that variables in the first block 
are more relevant than those in the second one, 
even if the corresponding coefficients are $1/2$ and $1$, respectively.
Conclusions from relevance by random permutations are different (top right panel): this method assigns larger relevance to variables in the second block than those in the first one, just because the corresponding coefficient is larger for block 2 than for block 1, without taking into account the dependence structure among the explanatory variables.

\begin{figure}[p]
\begin{center}
\hspace*{-1cm}
\includegraphics[scale=.85]{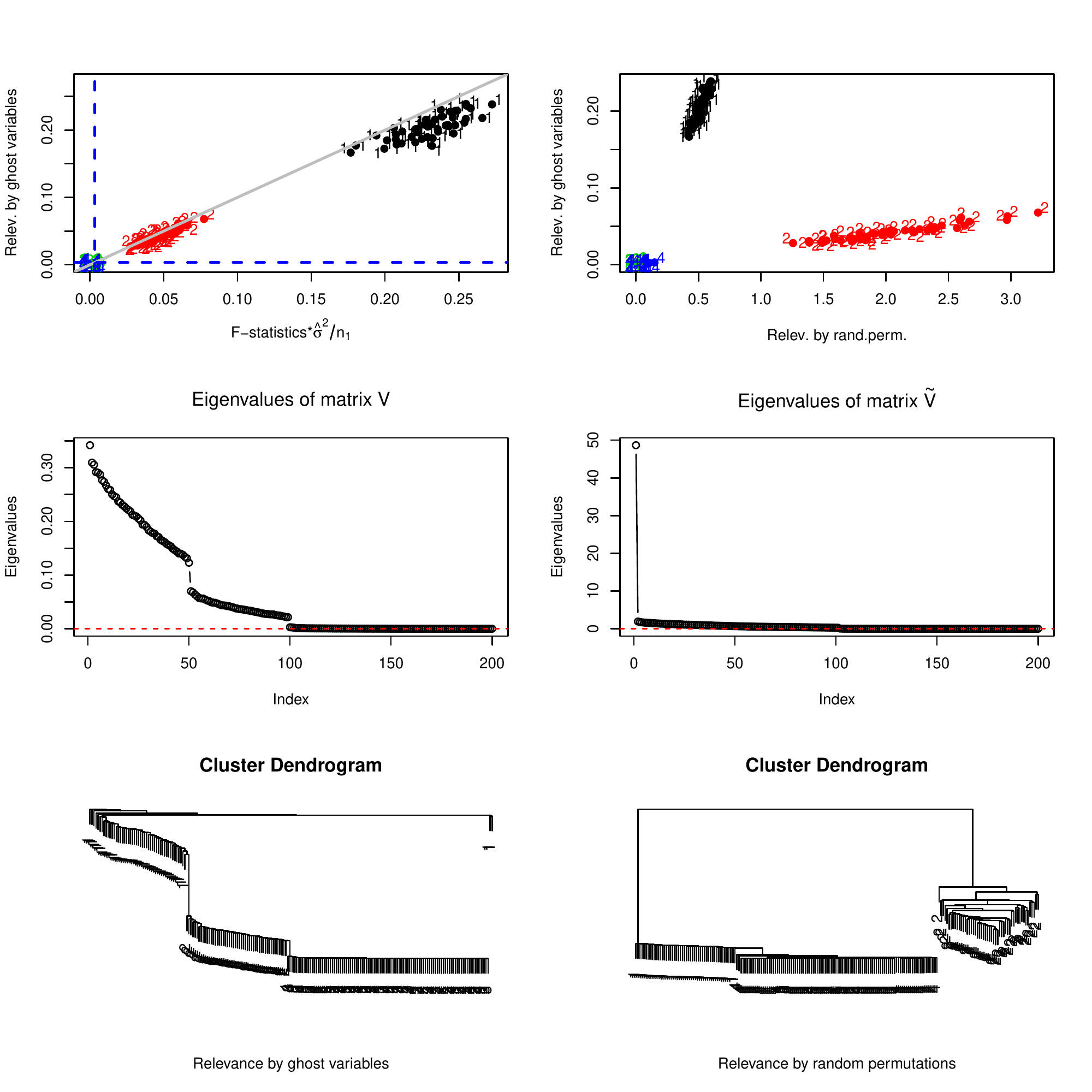}
\end{center}
\caption{Variable relevance in Example 3.} 
\label{fig:Ex_3_200}%
\end{figure}

Panels in the second row of Figure \ref{fig:Ex_3_200} show the eigenvalues of relevance matrices 
$\bV$ (left) and $\tilde{\mathbf{V}}$ (right).
In the left panel they are marked steps at eigenvalues $50$ and $99$.
The first 50 eigenvalues correspond to eigenvectors that are linear combinations of the first 50 columns of matrix $\bA$ (first block of explanatory variables), 
and the following 49 eigenvalues correspond to columns $51$ to $100$ (second block).
There are two additional eigenvectors (179 and 180, not reported here, with eigenvalues very close to zero) having almost constant weights in the 50 variables of the second block (and also significant weight in a few variables of blocks 3 and 4): they indicate that the sum of columns 51 to 100 of matrix $\bA$ is almost constant.
(This is coherent with the, not reported here, spectral analysis of the covariance of columns 51 to 100 in matrix $\bA$: the last eigenvalue is almost 0, much smaller than the 49th, and the corresponding eigenvector has almost constant coordinates).
Finally, the eigenvalues 100 to 200 are almost null. They correspond mainly to variables in blocks 3 and 4, that are irrelevant for $Y$. 

In the central right panel we can see that there is a main eigenvalue in matrix $\tilde{\mathbf{V}}$, that corresponds to an eigenvector (not reported here) that is proportional to the sum of the columns 51 to 100 of matrix $\tilde{\mathbf{A}}$. The eigenvalues 2 to 101 correspond to eigenvectors (not reported here) that are linear combinations of several of the first 100 columns of $\tilde{\mathbf{A}}$. The remain eigenvalues (almost null) correspond to variables in blocks 3 and 4.

The last two plots in Figure \ref{fig:Ex_3_200} show the dendograms corresponding to two hierarchical clustering of the 200 explanatory variables. In the first one (bottom left panel) the squared elements of relevance matrix $\bV$ are taken as similarities between variables, while matrix $\tilde{ \mathbf{V} }$ has been used in the second one (bottom right panel). Both dendograms are able to detect blocks of variables 1 and 2, while blocks 3 and 4 appear as a unique cluster.

\subsection{A real data example: Rent housing prices}

\label{sec:RealData} We present a real data example on rental housing, coming
from the Spanish real estate search portal Idealista (\url{www.idelista.com})
which allows customers to search for housing based on several criteria
(geography, number of rooms, price, and various other filters) among the
offers posted by other customers. We started working from the data set
downloaded from \url{www.idelista.com} by Alejandro Germ\'an-Serrano on
February 27th 2018
(available at \url{https://github.com/seralexger/idealista-data}; accessed
April 12th 2019). This data set contained 67201 rows (posted rental offers
actives at the download date) and 19 attributes, corresponding to all cities
in Spain. Some offers where activated for the first time in 2010.

We have selected posts corresponding to Madrid and Barcelona (16480 rows) and
finally work with 17 variables (some of them already in the original data set,
other calculated from the original attributes) listed in Table
\ref{tb:var_rent_hous}.


\begin{table}[p]
\caption{List of variables used in the rent housing prices examples. The
response variable is \texttt{price} (in logarithms), and other are the
explanatory variables. The district price level indicator \texttt{categ.distr}
has been computed in Barcelona and Madrid separately. For each district (there
are $N_{B}=10$ districts in Barcelona and $N_{M}=21$ in Madrid) the third
quartile of \texttt{price} is computed. Then these $N$ values (where $N=N_{B}$
or $N_{M}$) are classified by their own quartiles, and values $-3$, $-1$, $1$
and $3$ are assigned to them accordingly. Finally this district value is used
to define \texttt{categ.distr} for all the houses in each district. }%
\label{tb:var_rent_hous}%
\begin{tabular}
[c]{rp{12cm}}
& \\
\texttt{price} & Monthly rental price, the response variable (in logs).\\
\texttt{Barcelona} & 1 for houses in Barcelona, 0 for those in Madrid.\\
\texttt{categ.distr} & An indication of the district price level, taking the
values $-3$, $-1$, $1$ and $3$. See the caption for details.\\
& \\
\texttt{type.chalet} & These 4 variables are the binarization of the\\
\texttt{type.duplex} & original variable \texttt{type} with 5 levels:\\
\texttt{type.penthouse} & flat (the most frequent), chalet, duplex,\\
\texttt{type.studio} & penthouse and studio.\\
& \\
\texttt{floor} & Floor where the house is located.\\
\texttt{hasLift} & 1 if the house has lift, 0 otherwise.\\
\texttt{floorLift} & \texttt{abs(floor)*(1-hasLift)}\\
\texttt{size} & Surface, in squared meters.\\
\texttt{exterior} & 1 if the house is exterior, 0 otherwise.\\
\texttt{rooms} & Number of bedrooms.\\
\texttt{bathrooms} & Number of bathrooms.\\
\texttt{hasParkingSpace} & 1 if the house has a parking space, 0 otherwise.\\
\texttt{ParkingInPrice} & 1 if the parking space is included in the price, 0
otherwise.\\
\texttt{log\_activation} & logarithm of the number of days since the first
activation of the post.\\
&
\end{tabular}
\end{table}

In order to predict the logarithm of prices as a function of the other 16
explanatory variables, we have fitted three predictive models: a linear
regression, an additive model, and a neural network. For each model, the
variables relevance has been computed by ghost variables and by random
permutations (a 70\% of the data are used as training set, and the rest as
test set). For both, the linear and the additive model, the standard output
offers a rich information about the statistical significance of each
explanatory variable. Then the relevance analysis represents a complementary
information, that in most cases confirms the standard one, although matrix
relevance can add new lights. The situation is different for the neural
network model: in this case the relevance analysis will provided genuine new
insights on the importance of the explanatory variables, or groups of them.

We show below the results for the linear model and for the neural network. The
results for the additive model are not included here (they are accessible as
online material) because they are not very different from those obtained in
the linear model.

\subsubsection*{Linear model}

Table \ref{tb:out.lm.price} shows the standard output of the fitted linear
model (we have used the R function \texttt{lm}). Only three variables are not
significant at level 0.001. There are 7 variables with $t$-value larger than
10 in absolute value: \texttt{Barcelona}, \texttt{categ.distr}, \texttt{floor}%
, \texttt{log.size} (this one being the most significant variable),
\texttt{rooms}, \texttt{bathrooms}, and \texttt{log\_activation}. The adjusted
coefficient of determination $R^{2}$ is $0.7599$.

\begin{table}[ptb]
\caption{Rent housing prices: Standard output of the linear model.}%
\label{tb:out.lm.price}
{\small 
\begin{verbatim}
## lm(formula = log(price) ~ ., data = rhBM.price[Itr, ])
## 
## Residuals:
##      Min       1Q   Median       3Q      Max 
## -1.72437 -0.17604 -0.02316  0.15692  1.45330 
## 
## Coefficients:
##                   Estimate Std. Error t value Pr(>|t|)    
## (Intercept)      3.8169658  0.0344596 110.766  < 2e-16 ***
## Barcelona        0.1126307  0.0052554  21.431  < 2e-16 ***
## categ.distr      0.1169468  0.0033806  34.593  < 2e-16 ***
## type.chalet     -0.0846942  0.0203106  -4.170 3.07e-05 ***
## type.duplex     -0.0177992  0.0151519  -1.175  0.24013    
## type.penthouse   0.0428160  0.0101282   4.227 2.38e-05 ***
## type.studio     -0.0762350  0.0139991  -5.446 5.27e-08 ***
## floor            0.0128181  0.0009696  13.220  < 2e-16 ***
## hasLift          0.0480363  0.0118432   4.056 5.02e-05 ***
## floorLift       -0.0013898  0.0044109  -0.315  0.75270    
## log.size         0.6186668  0.0090654  68.245  < 2e-16 ***
## exterior        -0.0372539  0.0068935  -5.404 6.64e-08 ***
## rooms           -0.0501949  0.0034204 -14.675  < 2e-16 ***
## bathrooms        0.1431973  0.0047167  30.359  < 2e-16 ***
## hasParkingSpace -0.0074934  0.0129971  -0.577  0.56426    
## ParkingInPrice  -0.0408757  0.0138863  -2.944  0.00325 ** 
## log_activation   0.0418803  0.0018552  22.574  < 2e-16 ***
## ---
## Signif. codes:  0 '***' 0.001 '**' 0.01 '*' 0.05 '.' 0.1 ' ' 1
## 
## Residual standard error: 0.2647 on 11519 degrees of freedom
## Multiple R-squared:  0.7602, Adjusted R-squared:  0.7599 
## F-statistic:  2282 on 16 and 11519 DF,  p-value: < 2.2e-16
\end{verbatim}
}
\end{table}

The relevance by ghost variables results are shown in Figure
\ref{fig:VarRlevIdealista_Gh}. We can see (first row, first column plot) that
the 7 most relevant variables are the seven we cited before (those with
largest $t$-values in absolute value). This is a consequence of the existing
relation (Theorem \ref{th:Rel_Gh}) between the relevance by ghost variables
and $F$-values (the squares of $t$-values) in the linear model. This plot
shows that \texttt{log.size} is the most relevant variable, followed by
\texttt{categ.distr}, \texttt{bathrooms}, \texttt{log\_activation}, and
\texttt{Barcelona}. The relevance of \texttt{rooms} and \texttt{floor} 
is much lower.

Regarding the analysis of the relevance matrix $\mathbf{V}$, only the
eigenvectors explaining more than 1\% of the total relevance are plotted. The
first eigenvector accounts for the 60\% of total relevance, and it is
associated with the size of houses (\texttt{log.size}, \texttt{bathrooms} and
\texttt{rooms}). The second eigenvector (15\% of total relevance) is mostly
related with the district price level (\texttt{categ.distr}) and with some
small interaction of \texttt{bathrooms} and \texttt{rooms}. The third
eigenvector is a combination of the six most relevant variables, with
\texttt{bathrooms} having the largest weight. The other 4 eigenvector (from
5th to 7th) seem to be associated mainly with one explanatory variable
(\texttt{log\_activation}, \texttt{Barcelona}, \texttt{floor} and
\texttt{rooms}, respectively) with little interaction with other.

\begin{figure}[p]
\begin{center}
\vspace*{-1cm} \hspace*{-1.5cm}
\includegraphics[scale=.42]{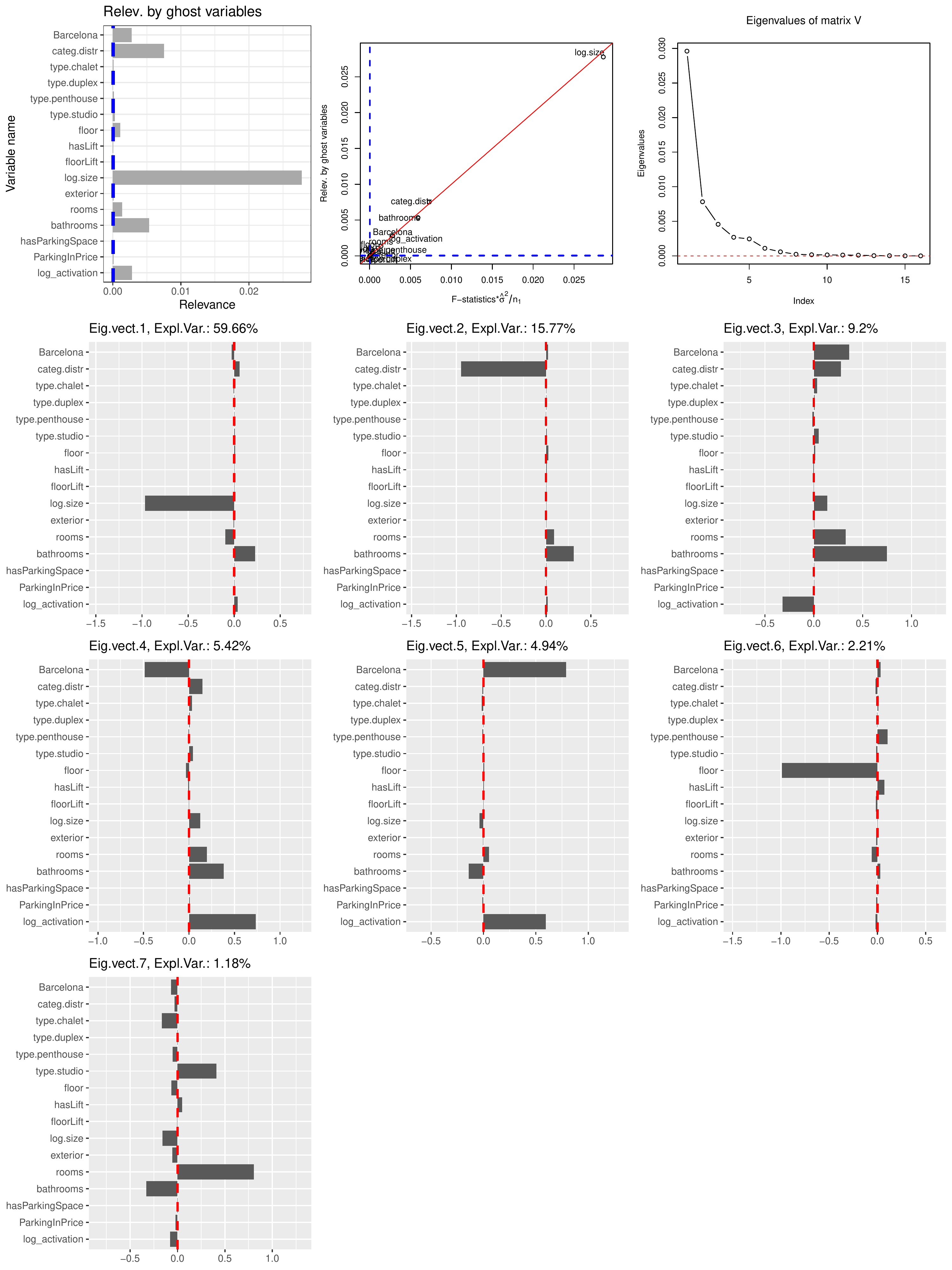}
\end{center}
\caption{Rent housing prices: Relevance by ghost variables for the linear
model. Only the eigenvectors explaining more than 1\% of the total relevance
are plotted.}%
\label{fig:VarRlevIdealista_Gh}%
\end{figure}

The analysis of relevance by random permutations (Figure
\ref{fig:VarRlevIdealista_RP}) agrees, in general terms, with the previous
one, although there are small differences. Now the most relevant variables are
\texttt{log.size} and \texttt{bathrooms}, with \texttt{categ.distr},
\texttt{rooms}, \texttt{log\_activation}, and \texttt{Barcelona} somehow
relegated. There are 6 eigenvectors responsible each for more than 1\% of the
total relevance. Each one is mainly associated with one of the previously
mentioned 6 explanatory variables, in the order of citation. The main
difference with the relevance by random permutations is that, when using
random permutations, no interactions between explanatory variables have been detected.

\begin{figure}[p]
\begin{center}
\hspace*{-1.5cm} \includegraphics[scale=.42]{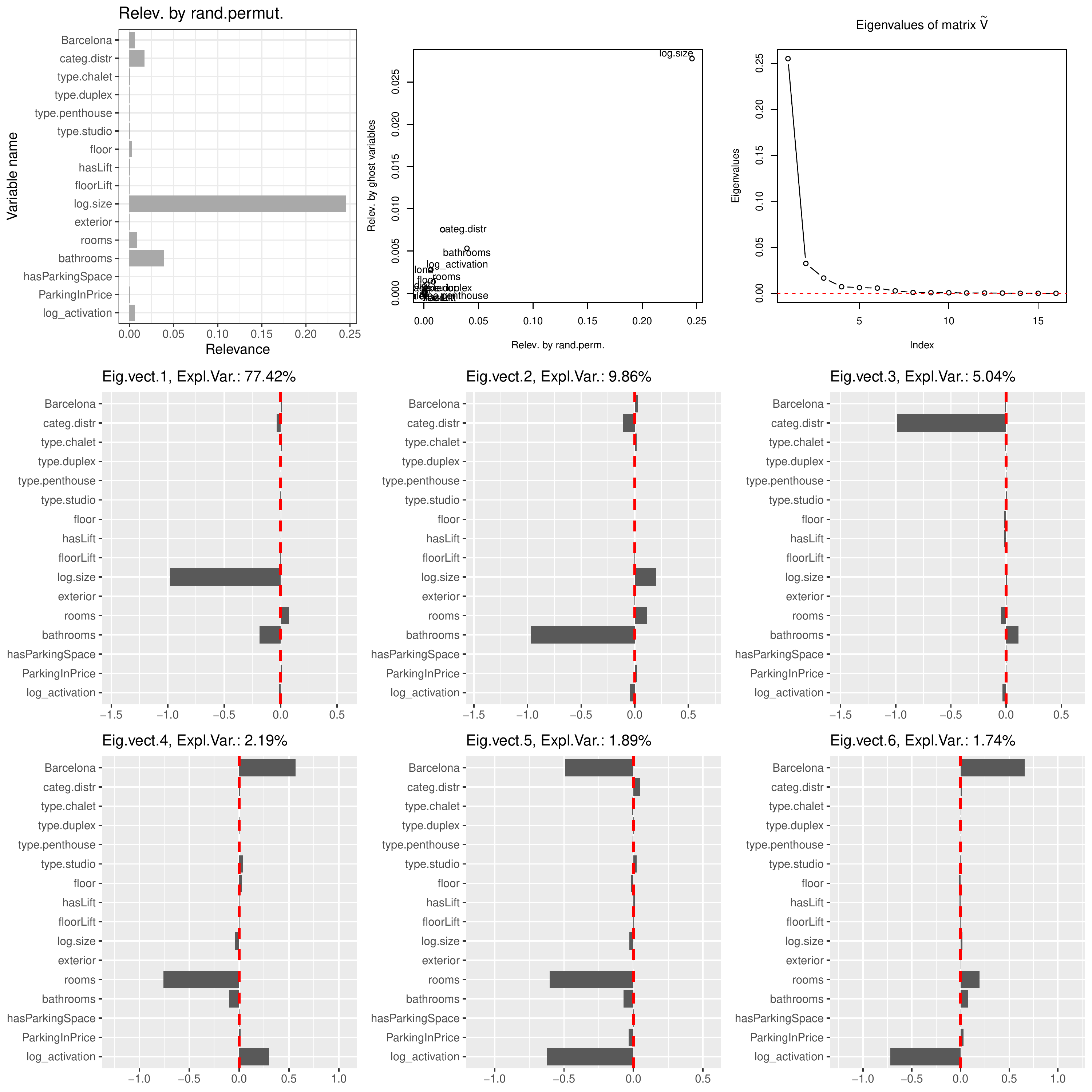}
\end{center}
\caption{Rent housing prices: Relevance by random permutations for the linear
model. Only the eigenvectors explaining more than 1\% of the total relevance
are plotted.}%
\label{fig:VarRlevIdealista_RP}%
\end{figure}

\subsubsection*{Neutral network}

A one-hidden-layer neural network has been fitted using the \texttt{nnet}
function from the R package \texttt{nnet} (\citeNP{nnet:2002}). The explanatory
variable were centered and scaled before the fitting. Tuning parameters,
\texttt{size} (number of neurons in the hidden layer) and \texttt{decay}
parameter, are chosen using \texttt{caret} (\citeNP{caret:2018}) by 10-fold cross
validation (in the training set). The candidates values for \texttt{size} were
10, 15, and 20, and they were 0, 0.1, 0.3, and 0.5 for \texttt{decay}. Finally
the chosen values were \texttt{size=10} and \texttt{decay=0.5}. With these
values, the whole training sample was used to fit the neural network and the
results were stored in the object \texttt{nnet.logprice}.

Table \ref{tb:out.nn.price} shows the little information obtained when
printing the output of the \texttt{nnet} function. Additionally, a quantity
similar to the coefficient of determination for this fit has been computed and
printed, with a value of $0.79$. It is a little bit larger than for the linear
and the additive models, so neural network provides the best fit in the sense
of maximum coefficient of determination. You can see that the output in Table
\ref{tb:out.nn.price} does not provide any insight about which explanatory
variables are more responsible for that satisfactory fit. Relevance measures
will be of help in this respect.

\begin{table}[ptb]
\caption{Rent housing prices: Output of the neural network model.}%
\label{tb:out.nn.price}
{\small 
\begin{verbatim}
# > nnet.logprice
#
# a 16-10-1 network with 181 weights
# inputs: Barcelona categ.distr type.chalet type.duplex type.penthouse  
# type.studio floor hasLift floorLift log.size exterior rooms bathrooms 
# hasParkingSpace ParkingInPrice log_activation 
#
# output(s): log(price) 
#
# options were - linear output units  decay=0.5
#
# > 1-(mean(nnet.logprice$residuals^2)/var(log(scaled.rhBM.price.tr$price)))
# [1] 0.7912618
\end{verbatim}
}
\end{table}

Results on relevance by ghost variables for the fitted neural network are
shown in Figure \ref{fig:VarRlevIdealista_nn_Gh}. In addition to the 7
variables that were relevant for the linear and additive models, now the 4
variables related with the type of house appear with small but not null
relevance. The same can be said for the variable \texttt{ParkingInPrice}.
Again, the two most relevant variables are \texttt{log.size} and
\texttt{categ.distr}, in this order.

There are 10 eigenvectors with percentages of explained relevance greater than
1\% (the first 9 are shown in Figure \ref{fig:VarRlevIdealista_nn_Gh}). The first eigenvector (47\% of total relevance) corresponds almost
entirely to \texttt{log.size}, and the second one (20\%) to
\texttt{categ.distr} (in fact, the first 3 eigenvectors are similar to those in the linear model, Figure \ref{fig:VarRlevIdealista_Gh}). The variables \texttt{bathrooms} and \texttt{rooms}
appear together in several eigenvectors, always accompanying other variables.
Similar joint behavior present \texttt{Barcelona} and \texttt{log\_activation}. 
Variables referring to different types of houses appear at eigenvectors 4th
to 9th (except the 6th one). The 7th eigenvector is mainly related with \texttt{floor}.

\begin{figure}[p]
\begin{center}
\vspace*{-.5cm} 
\includegraphics[scale=.42]{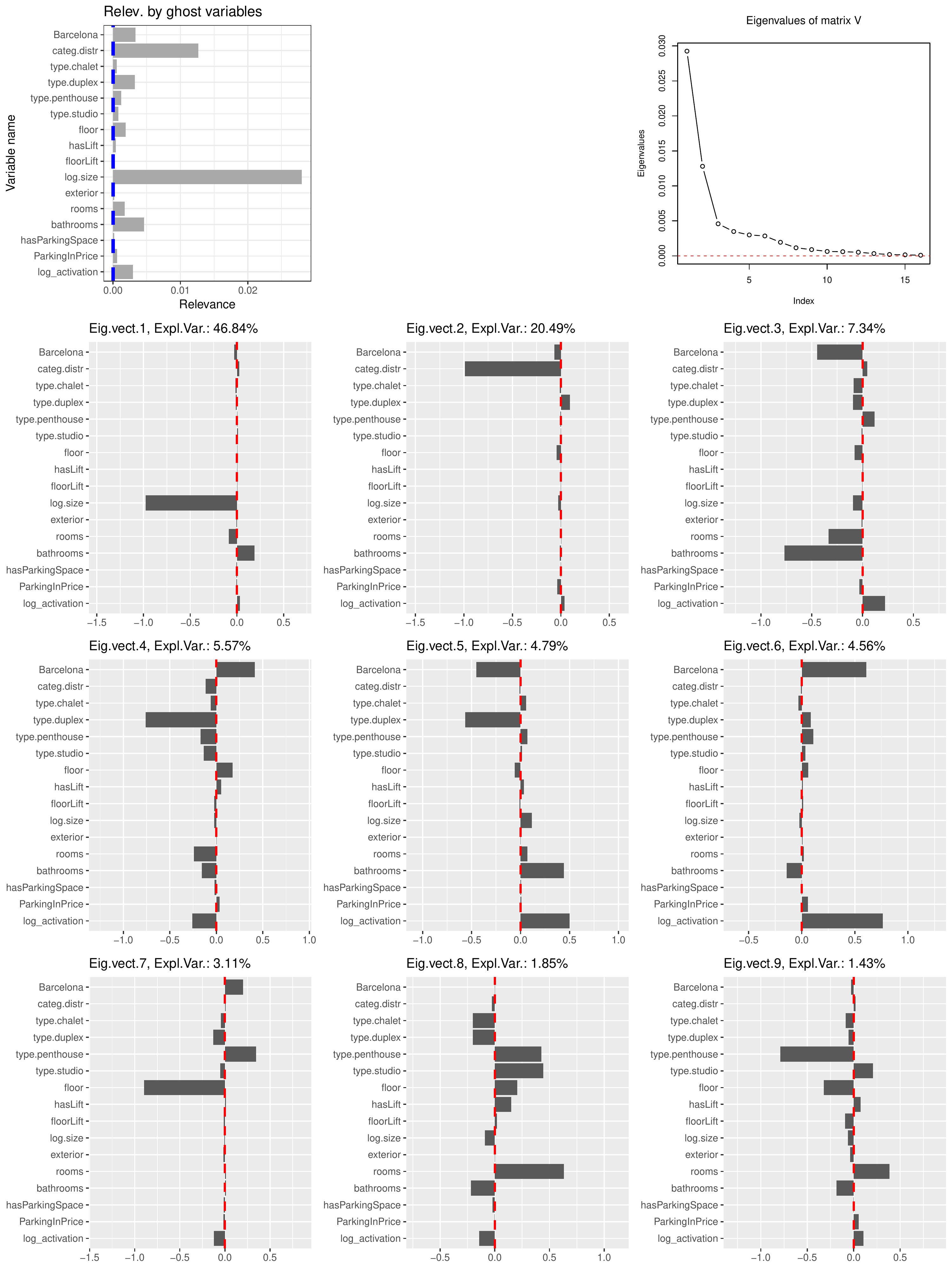}
\end{center}
\caption{ Relevance by ghost variables for the neural network model. }%
\label{fig:VarRlevIdealista_nn_Gh}
\end{figure}

Let us analyze now the relevance by random permutations (Figure \ref{fig:VarRlevIdealista_nn_RP}). In this case all the
variables, except \texttt{type.chalet}, has non-null relevance. The results
are quite different from what we have seen before. Now the most relevant
variable is \texttt{type.penthouse}, followed first by \texttt{floor} and
\texttt{bathrooms}, and then by \texttt{log\_activation},
\texttt{ParkingInPrice}, \texttt{rooms}, \texttt{categ.distr},
\texttt{Barcelona}, and \texttt{type.studio}. Interestingly, \texttt{log.size}
has a very low relevance now, contrary to what we have seen before.

Variables \texttt{type.penthouse}, \texttt{floor} and \texttt{bathrooms}
dominate the first 3 eigenvectors of the relevance matrix $\tilde{ \mathbf{V}
}$ (jointly accounting for almost 60\% of the total relevance). 
Eigenvectors 4 to 9 
seem to be each related with just one explanatory variable.

\begin{figure}[p]
\begin{center}
\vspace*{-.5cm} 
\includegraphics[scale=.42]{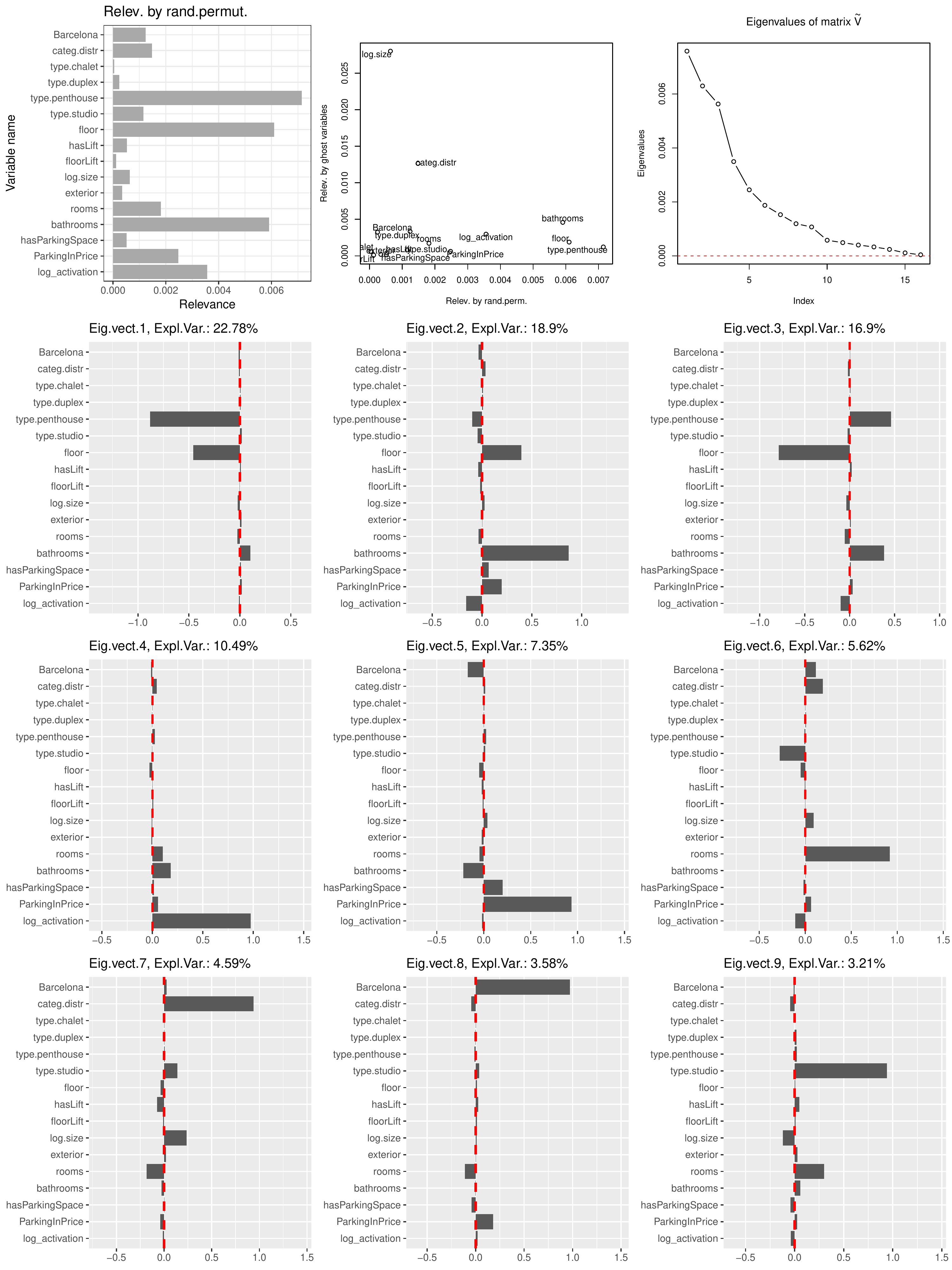}
\end{center}
\caption{Relevance by random permutations for the neural network model.}%
\label{fig:VarRlevIdealista_nn_RP}
\end{figure}

When computing relevances by Ghost variables, the most relevant one is 
{\tt log.size},
while by random permutations are 
{\tt type.penthouse}, {\tt floor} and  {\tt bathrooms}.
In order to determine which of these two methods 
better determines the really important variables, we compute the relevance by omission of these four variables removing them (one at a time) from the set of explanatory variables and then training again the neural network (observe that this computations are time consuming, because the optimal tuning parameters have to be selected each time we remove an explanatory variable). 
The most important variables are those with a larger relevance by omission. 
Table \ref{tb:VarRlevIdealista_nn} shows these values. 
It follows that the most important variable (according to relevance by omission) is  {\tt log.size}, the same result obtained whit relevance by Ghost variables. 
The other variables (that were labeled as relevant by random permutations) have a very low value of relevance by omission, compared with {\tt log.size}.
The last row of Table \ref{tb:VarRlevIdealista_nn} contains the relevance by omission when removing simultaneously the three most relevant variables according to random permutations. It follows that the relevance by omission of this group of variables is less than 32\% of the  relevance by omission of {\tt log.size}.
%

\begin{table}
\caption{
Relevance by omission of several explanatory variables and a group of them when fitting neural network models. 
}
\label{tb:VarRlevIdealista_nn}
\begin{center}
\begin{tabular}{r l}
Omitted variables & Relevance by omission\\ \hline 
{\tt log.size}         & .0321 \\
{\tt type.penthouse}   & .0046 \\
{\tt floor}            & .0057 \\
{\tt bathrooms}        & .0085 \\
{\tt type.penthouse}, {\tt floor} and {\tt bathrooms} & .0102 \\ \hline 
\end{tabular}
\end{center}
\end{table}

Summarizing the relevance results for the neural network, it can be said that
they present certain differences with respect to linear and additive models.
First, there are completely different results with ghost variables and with
random permutation. Second, the variables found as most important with random
permutation were among the less important in linear and additive
models. Third in the relevance by ghost variables more variables are
considered but the order of the variables by relevance does not change with
respect to the results of linear and additive models. 
Fourth, the random permutation approach does not seem to be useful for NN in this example, whereas the ghost variable approach find the same results as relevance by omission in a much more computing efficiency.

\section{Conclusions and further research}

\label{sec:concl}

We have defined the relevance of a variable in a complex model by its
contribution to out of sample prediction and proposed a new way to measure
this contribution: to compare the predictions of the model that includes this
variable to those of a model in which this variable is substituted by its
ghost variable, that is, an estimation of the variable by using the other
explanatory variables. We have also shown that this approach is more useful
than usual procedures of deleting the variable or using its random
permutation. We have proved that in linear regression this approach is related
to the $F$-statistic used to check the significance of the variable and
therefore the computation of the relevance of a variable in a complex
prediction model is an extension of the significance concept to other models
in which this concept is not usually considered. 
With many dependent variables,
the relevance of a variable by itself is less useful and it is important to
consider the joint contribution of sets of variables to the out of sample
prediction. We have introduced the relevance matrix as a way to find groups of
related variables with important joint effects in out of sample
prediction. In the linear model, we have proved the relationship between the
relevance matrix and the matrix of partial correlations of the explanatory variables.

The research in this paper has open many problems for further investigation.
As was indicated in the Introduction, the ghost variables are related to the
knockoff variables for controlling the false discovery rate in large
regression problems and this possible connection and the use of the
relevance matrix as a screening device requires further research.
Also, there
are some aspects in the application of these ideas to Neural Network
prediction function that should be further analyzed. First, the estimation of
the ghost variables can be also carried out by some NN models, to allow for
non linear relations among the explanatory variables. See for instance
\shortciteN{sanchezimproving} for the estimation of missing values in NN.
Second, a comparison of this approach to the local sensitivity analysis
looking at the derivatives of the prediction function, that is often used in
NN, will be important. Third, the extensions of these ideas to binary (or
multi class) classification is an important problem for discrimination.

\section*{Acknowledgements}
This research was supported by the Spanish Ministerio de Econom\'{i}a, 
  Industria y Competitividad (MINECO), and Fondo Europeo de Desarrollo
  Regional (FEDER, UE) grants MTM2017-88142-P (Pedro Delicado) and
  ECO2015-66593-P (Daniel Pe\~na). We are grateful to Alejandro Germ\'an-Serrano 
  for providing us with access to Idealista data.

\bigskip
\begin{center}
{\large\bf SUPPLEMENTARY MATERIAL}
\end{center}

\begin{description}


\item[Additional results and data analysis:] Results on relevance by omission, and additive model for the real data example. (See at the end of this .pdf file)

\item[R-scripts and datasets:] 
The R-scripts containing code to reproduce the computations and graphics in the article
can be found at \url{https://github.com/pedrodelicado/GhostVariables}.
The datasets used as examples in the article are also available there. 

\end{description}

\bibliographystyle{Chicago}
\bibliography{RelevMatrix}

\appendix

\section{Appendix}

\label{sec:appendix}

\subsection{Updating formula for the OLS estimator}

\label{Ap:updating_formula} Let $\mathbf{b}\in\mathbb{R}^{p}$, $c\in
\mathbb{R}$, and $\mathbf{A}\in\mathbb{R}^{p\times p}$, an invertible matrix.
The expression of the inverse of an invertible block matrix is as follows:
\[
\left(
\begin{array}
[c]{cc}%
\mathbf{A} & \mathbf{b}\\
\mathbf{b}^{T} & c
\end{array}
\right)  ^{-1}=\left(
\begin{array}
[c]{cc}%
\mathbf{A}^{-1}+\frac{1}{k}\mathbf{A}^{-1}\mathbf{b}\mathbf{b}^{T}%
\mathbf{A}^{-1} & -\frac{1}{k}\mathbf{A}^{-1}\mathbf{b}\\
-\frac{1}{k}\mathbf{b}^{T}\mathbf{A}^{-1} & \frac{1}{k}%
\end{array}
\right)  ,
\]
where $k=c-\mathbf{b}^{T}\mathbf{A}^{-1}\mathbf{b}$.

Consider the linear regression with responses in vector $\mathbf{y}
\in\mathbb{R} ^{n}$ and regression matrix $( \mathbf{X} , \mathbf{z}
)\in\mathbb{R} ^{n\times(p+1)}$. The OLS estimated regression coefficients are
given by
\[
{\binom{\hat{\beta}_{x}}{\hat{\beta}_{z}}}=\left(
\begin{array}
[c]{cc}%
\mathbf{X}^{T}\mathbf{X} & \mathbf{X}^{T}\mathbf{z}\\
\mathbf{z}^{T}\mathbf{X} & \mathbf{z}^{T}\mathbf{z}%
\end{array}
\right)  ^{-1}{\binom{\mathbf{X}^{T}}{\mathbf{z}^{T}}}\mathbf{y}%
\]
Then, using the formula for the inverse of a block matrix, we have that
\[
{\binom{\hat{\beta}_{x}}{\hat{\beta}_{z}}}=\left(
\begin{array}
[c]{cc}%
(\mathbf{X}^{T}\mathbf{X})^{-1}+\frac{1}{k}(\mathbf{X}^{T}\mathbf{X}%
)^{-1}\mathbf{X}^{T}\mathbf{z}\mathbf{z}^{T}\mathbf{X}(\mathbf{X}%
^{T}\mathbf{X})^{-1} & -\frac{1}{k}(\mathbf{X}^{T}\mathbf{X})^{-1}%
\mathbf{X}^{T}\mathbf{z}\\
-\frac{1}{k}\mathbf{z}^{T}\mathbf{X}(\mathbf{X}^{T}\mathbf{X})^{-1} & \frac
{1}{k}%
\end{array}
\right)  {\binom{\mathbf{X}^{T}\mathbf{y}}{\mathbf{z}^{T}\mathbf{y}}}%
\]
with
\[
k=\mathbf{z}^{T}\mathbf{z}-\mathbf{z}^{T}\mathbf{X}(\mathbf{X}^{T}%
\mathbf{X})^{-1}\mathbf{X}^{T}\mathbf{z}= \mathbf{z}^{T}\mathbf{z} -
\hat{\mathbf{z}}_{x}^{T}\hat{\mathbf{z}}_{x}= (\mathbf{z} - \hat{\mathbf{z}%
}_{x})^{T}(\mathbf{z} - \hat{\mathbf{z}}_{x}),
\]
where $\hat{\mathbf{z}}_{x}=\mathbf{H}_{x}\mathbf{z}$ and $\mathbf{H}%
_{x}=\mathbf{X}(\mathbf{X}^{T}\mathbf{X})^{-1}\mathbf{X}^{T}$ is the \emph{hat
matrix} in any linear regression over $\mathbf{X}$. Then, calling $\hat{\beta
}_{0}=(\mathbf{X}^{T}\mathbf{X})^{-1}\mathbf{X}^{T}\mathbf{y}$, $\hat
{\mathbf{\alpha}}=(\mathbf{X}^{T}\mathbf{X})^{-1}\mathbf{X}^{T}\mathbf{z,}$ we
have
\[
{\binom{\hat{\beta}_{x}}{\hat{\beta}_{z}}}=\left(
\begin{array}
[c]{c}%
\hat{\beta}_{0}+\frac{1}{k}\hat{\mathbf{\alpha}}\hat{\mathbf{z}}_{x}%
^{T}\mathbf{y}-\frac{1}{k}\hat{\mathbf{\alpha}}\mathbf{z}^{T}\mathbf{y}\\
\frac{1}{k}(\mathbf{z-\hat{\mathbf{z}}_{x})}^{T}\mathbf{y}%
\end{array}
\right)
\]
and finally
\[
\hat{\beta}_{x}=\hat{\beta}_{0} -\hat{\mathbf{\alpha}}\hat{\beta}_{z}.
\]
Therefore, the following updating formula is derived:
\[
\hat{ \mathbf{y} }_{x,z} = \mathbf{X} \hat{\beta}_{x} + \mathbf{z} \hat{\beta
}_{z} = \mathbf{X} \hat{\beta}_{0} - ( \mathbf{X} \hat{\alpha}) \hat{\beta
}_{z} + \mathbf{z} \hat{\beta}_{z} = \hat{ \mathbf{y} }_{x} + ( \mathbf{z}
-\hat{ \mathbf{z} }_{x}) \hat{\beta}_{z}.
\]



\subsection{Proof of Theorem \ref{th:Rel_Gh}}

\label{Ap:Proof_Rel_Gh}

By definition, the \emph{relevance by a ghost variable} of the variable $Z$
is
\[
\mbox{Rel}_{\mbox{\scriptsize Gh}}(Z) = \frac{1}{n_{2}} (\hat{ \mathbf{y}
}_{2.X.z} - \hat{ \mathbf{y} }_{2.X.\hat{z}})^{T} (\hat{ \mathbf{y} }_{2.X.z}
- \hat{ \mathbf{y} }_{2.X.\hat{z}}) =
\]
\[
\frac{1}{n_{2}} \hat{\beta}_{z} ( \mathbf{z} _{2}-\hat{ \mathbf{z} }%
_{2.2})^{T} ( \mathbf{z} _{2}-\hat{ \mathbf{z} }_{2.2}) \hat{\beta}_{z} =
\frac{1}{n_{2}} \hat{\beta}_{z}^{2} k_{\mbox{\scriptsize Gh}} = \frac
{\hat{\sigma}^{2}F_{z}}{n_{1}} \frac{k_{\mbox{\scriptsize Gh}}/n_{2}}{k/n_{1}%
},
\]
where $k_{\mbox{\scriptsize Gh}}=( \mathbf{z} _{2}-\hat{ \mathbf{z} }
_{2.2})^{T} ( \mathbf{z} _{2}-\hat{ \mathbf{z} }_{2.2})$ and $k$ has been
defined in Appendix \ref{Ap:updating_formula}. The proof concludes by
observing that $\hat{\sigma}^{2}_{z.x,n_{2}}=k_{\mbox{\scriptsize Gh}}/n_{2}$
is an estimator of the residual variance in the linear regression model
$Z=X^{T}\alpha+ \varepsilon_{z}$, as they also are $k/(n_{1}-p)$ and
$\hat{\sigma}^{2}_{z.x,n_{1}}=k/n_{1}=((n_{1}-p)/n_{1}) (k/(n_{1}-p))$. The
expression involving the $O_{p}$ notation is derived by standard arguments for
the limit of a quotient.

\subsection{Proof of Theorem \ref{th:relevance_matrix}}

\label{Ap:Proof_relevance_matrix}

\begin{lemma}
\label{lemma_a_b} Let $\mathbf{a} $ and $\mathbf{b} $ be two non-null vectors
of $\mathbb{R} ^{d}$. Let $\mathbb{P} _{ \mathbf{b} }( \mathbf{a} )$ be the
projection vector of $\mathbf{a} $ over $\mathbf{b} $, and let $\alpha(
\mathbf{a} , \mathbf{b} )$ be the angle between $\mathbf{a} $ and $\mathbf{b}
$. Then
\[
\cos\left(  \alpha\left(  \mathbf{a} - \mathbb{P} _{ \mathbf{b} }( \mathbf{a}
), \mathbf{b} - \mathbb{P} _{ \mathbf{a} }( \mathbf{b} )\right)  \right)  = -
\cos(\alpha( \mathbf{a} , \mathbf{b} )).
\]

\end{lemma}

\begin{proof}
Given that $\cos( \alpha(\ba,\bb))=\ba^T\bb/(\|\ba\|\,\|\bb\|)$
and
$\PP_{\bb}(\ba)
= \cos( \alpha(\ba,\bb)) \|\ba\| (\bb/\|\bb\|)
= (\ba^T \bb) \bb / \|\bb\|^2
$,
it follows that
\[
\ba^T \PP_{\ba}(\bb)= \PP_{\bb}(\ba)^T \bb=
\ba^T\bb,
\,\,
\PP_{\bb}(\ba)^T \PP_{\ba}(\bb)=
\cos^2(\alpha(\ba,\bb))\, \ba^T\bb
\]
and
\[
\|\PP_{\bb}(\ba)\|^2=\PP_{\bb}(\ba)^T \PP_{\bb}(\ba)=
\cos^2(\alpha(\ba,\bb))\|\ba\|^2.
\]
By the Pythagoras Theorem,
\[
\| \ba - \PP_{\bb}(\ba) \|^2=\|\ba\|^2 -  \|\PP_{\bb}(\ba)\|^2
=\sin^2(\alpha(\ba,\bb)) \|\ba\|^2.
\]
Finally,
\[
\cos\left(\alpha\left(\ba-\PP_{\bb}(\ba), \bb-\PP_{\ba}(\bb)\right)\right) =
\frac{(\ba - \PP_{\bb}(\ba))^T(\bb - \PP_{\ba}(\bb))}
{\| \ba - \PP_{\bb}(\ba) \|\| \bb - \PP_{\ba}(\bb) \|}=
\]
\[
\frac{\ba^T\bb - \PP_{\bb}(\ba)^T\bb -\ba^T \PP_{\ba}(\bb) +
\PP_{\bb}(\ba)^T \PP_{\ba}(\bb)}
{\sin^2(\alpha(\ba,\bb)) \|\ba\|\|\bb\|}
=
\]
\[
\frac{-(1-\cos^2(\alpha(\ba,\bb)))\ba^T\bb}
{\sin^2(\alpha(\ba,\bb)) \|\ba\|\|\bb\|}
=
- \cos(\alpha(\ba,\bb)).
\]
\end{proof}

\subsubsection*{Proof of Theorem \ref{th:relevance_matrix}}

We start proving that the matrix
\[
\mathbf{G}=\frac{1}{n_{2}}(\mathbf{X}_{2}-\hat{\mathbf{X}}_{2}
)^{\mbox{\scriptsize T}}(\mathbf{X}_{2}-\hat{\mathbf{X}}_{2})
\]
has generic non-diagonal element $g_{jk}=\hat{\rho}_{jk.R}\hat{\sigma}_{[j]}
\hat{\sigma}_{[k]}$ for $j\neq k$, where $\hat{\rho}_{jk.R}$ is the partial
correlation coefficient between variables $j$ and $k$ when controlling by the
rest of variables, and $\hat{\sigma}_{[j]}^{2}$ is the $j$-th element in the
diagonal of $G$:
\[
g_{jj}=\hat{\sigma}_{[j]}^{2}=\frac{1}{n_{2}} (\mathbf{x}_{2.j}-\hat
{\mathbf{x}}_{2.j})^{\mbox{\scriptsize T}}(\mathbf{x}_{2.j}-\hat{\mathbf{x}%
}_{2.j}).
\]
It is equivalent to proof that
\[
\hat{\rho}_{jk.R}=-\frac{g_{jk}}{\sqrt{g_{jj}g_{kk}}}= - \frac{(\mathbf{x}%
_{2.j}-\hat{\mathbf{x}}_{2.j})^{\mbox{\scriptsize T}} (\mathbf{x}_{2.k}%
-\hat{\mathbf{x}}_{2.k})} {\sqrt{(\mathbf{x}_{2.j}-\hat{\mathbf{x}}%
_{2.j})^{\mbox{\scriptsize T}} (\mathbf{x}_{2.j}-\hat{\mathbf{x}}_{2.j})}
\sqrt{(\mathbf{x}_{2.k}-\hat{\mathbf{x}}_{2.k})^{\mbox{\scriptsize T}}
(\mathbf{x}_{2.k}-\hat{\mathbf{x}}_{2.k})}},
\]
that is, we have to prove that the cosinus of the angle between the vector of
residuals $\mathbf{x}_{2.j}-\hat{\mathbf{x}}_{2.j}$ and $\mathbf{x}_{2.k}%
-\hat{\mathbf{x}}_{2.k}$ is equal to minus the cosinus of the angle between
the vector of residuals $\mathbf{a} =\mathbf{x}_{2.j}-\hat{\mathbf{x}}%
_{2.j.R}$ and $\mathbf{b} =\mathbf{x}_{2.k}-\hat{\mathbf{x}}_{2.k.R}$,
obtained when regressing $\mathbf{x} _{2.j}$ and $\mathbf{x} _{2.k}$
respectively over $\mathbf{R} =\mathbf{X}_{2.[jk]}$, the matrix with columns
$\mathbf{x}_{2.1},\ldots,\mathbf{x}_{2.p}$, except $\mathbf{x}_{2.j}$ and
$\mathbf{x}_{2.k}$.

We use now the notation $\mathbb{P} _{ \mathbf{U} }( \mathbf{x} )$ to denote
the projection of the vector $\mathbf{x} $ over the linear space $\mathbf{U}
$, and $\{ \mathbf{R} , \mathbf{x} \}$ for the subspace generated by the
columns of the matrix $\mathbf{R} $ and the vector $\mathbf{x} $. Observe
that
\[
\hat{ \mathbf{x} }_{2.j}= \mathbb{P} _{\{ \mathbf{R} , \mathbf{x} _{2.k}\}}(
\mathbf{x} _{2.j})= \mathbb{P} _{\{ \mathbf{R} , \mathbf{b} \}}( \mathbf{x}
_{2.j})= \mathbb{P} _{\{ \mathbf{R} , \mathbf{b} \}}(\hat{\mathbf{x}}_{2.j.R}+
\mathbf{a} )=
\]
\[
\mathbb{P} _{\{ \mathbf{R} , \mathbf{b} \}}(\hat{\mathbf{x}}_{2.j.R})+
\mathbb{P} _{\{ \mathbf{R} , \mathbf{b} \}}( \mathbf{a} )= \hat{\mathbf{x}%
}_{2.j.R} + \mathbb{P} _{ \mathbf{b} }( \mathbf{a} ).
\]
Therefore
\[
\mathbf{x}_{2.j}-\hat{ \mathbf{x} }_{2.j}= \mathbf{x}_{2.j}-\hat{\mathbf{x}%
}_{2.j.R} - \mathbb{P} _{ \mathbf{b} }( \mathbf{a} ) = \mathbf{a} - \mathbb{P}
_{ \mathbf{b} }( \mathbf{a} ).
\]
Analogously, $\mathbf{x}_{2.k}-\hat{ \mathbf{x} }_{2.k}= \mathbf{b} -
\mathbb{P} _{ \mathbf{a} }( \mathbf{b} )$. A direct application of Lemma
\ref{lemma_a_b} finishes the proof that $\hat{\rho}_{jk.R}=-g_{jk}%
/\sqrt{g_{jj}g_{kk}}$. The other statements in the Theorem follow directly
from the relation $\mathbf{V} =\mbox{diag}(\hat{\beta}) \mathbf{G}
\mbox{diag}(\hat{\beta})$.

\newpage

\section*{Online Supplement}

\subsubsection*{Supplement A: Results on relevance by omission}

\begin{proposition}
\label{prop:Rel_Om_Train} Let
\[
\mbox{\rm Rel}_{\mbox{\scriptsize\rm Om}}^{\mbox{\scriptsize\rm Train}}%
(Z)=\frac{1}{n_{1}}(\hat{\mathbf{y}}_{1.X.z}-\hat{\mathbf{y}}_{1.X})^{T}%
(\hat{\mathbf{y}}_{1.X.z}-\hat{\mathbf{y}}_{1.X})
\]
be the relevance by omission of the variable $Z$, evaluated in the training
sample. It happens that
\[
\frac{n_{1}}{\hat{\sigma}^{2}}\mbox{Rel}_{\mbox{\scriptsize Om}}%
^{\mbox{\scriptsize Train}}(Z)=F_{z},
\]
where $\hat{\sigma}^{2}=(\mathbf{y}_{1}-\hat{\mathbf{y}}_{1.X.z}%
)^{T}(\mathbf{y}_{1}-\hat{\mathbf{y}}_{1.X.z})/(n_{1}-p-1)$. Moreover,
\[
\mbox{Rel}_{\mbox{\scriptsize Om}}^{\mbox{\scriptsize Train}}(Z)=\hat{\beta
}_{z}^{2}\hat{\sigma}_{z.x,n_{1}}^{2},
\]
where $\hat{\sigma}_{z.x,n_{1}}^{2}$ is a consistent estimator of
$\sigma_{z.x}^{2}$, the residual variance in the model $Z=X^{T}\alpha
+\varepsilon_{z}$, computed from the training sample.
\end{proposition}

\begin{proof}
We are dealing with the estimation of models $m(x,z)=x^{T}\beta_{x}+z\beta
_{z}$ and $m_{p}(x)=x^{T}\beta_{0}$ from the training set $(\mathbf{X}%
_{1},\mathbf{z}_{1},\mathbf{y}_{1})$. The ordinary least squares (OLS)
coefficient estimators are $\hat{\beta}_{0}=(\mathbf{X}_{1}^{T}\mathbf{X}%
_{1})^{-1}\mathbf{X}_{1}^{T}\mathbf{y}_{1}$ and
\[
{\binom{\hat{\beta}_{x}}{\hat{\beta}_{z}}}=\left(
\begin{array}
[c]{cc}%
\mathbf{X}_{1}^{T}\mathbf{X}_{1} & \mathbf{X}_{1}^{T}\mathbf{z}_{1}\\
\mathbf{z}_{1}^{T}\mathbf{X}_{1} & \mathbf{z}_{1}^{T}\mathbf{z}_{1}%
\end{array}
\right)  ^{-1}{\binom{\mathbf{X}_{1}^{T}}{\mathbf{z}_{1}^{T}}}\mathbf{y}_{1}.
\]
The predicted values for the response are, respectively, $\hat{\mathbf{y}%
}_{1.X.z}=\mathbf{X}_{1}\hat{\beta}_{x}+\mathbf{z}_{1}\hat{\beta}_{z}$ and
$\hat{\mathbf{y}}_{1.X}=\mathbf{X}_{1}\hat{\beta}_{0}$. Using the expression
of the inverse of a partitioned matrix (see Appendix A.1
), it is easy to obtain the well known result%
\begin{equation}
\hat{\beta}_{x}=\hat{\beta}_{0}-\hat{\alpha}_{1}\hat{\beta}_{z}%
\label{betayalfa}%
\end{equation}
where $\hat{\alpha}_{1}=(\mathbf{X}_{1}^{T}\mathbf{X}_{1})^{-1}\mathbf{X}%
_{1}^{T}\mathbf{z}_{1}$ is the vector of regression coefficients in the
regression of the omitted variable on the other explanatory variables in the
training set, and
\begin{equation}
\hat{\beta}_{z}=\frac{1}{k}(\mathbf{z}_{1}-\hat{\mathbf{z}}_{1})^{T}%
\mathbf{y}_{1}\label{betaz}%
\end{equation}
where $\hat{\mathbf{z}}_{1}=\mathbf{X}_{1}\hat{\alpha}_{1}$ and
$k=(\mathbf{z}_{1}-\hat{\mathbf{z}}_{1})^{T}(\mathbf{z}_{1}-\hat{\mathbf{z}}_{1})$. 
These results (see, e.g., \citeNP{seber2003linear}) show that the
multiple regression coefficient of each variable, $\hat{\beta}_{z}$, is the
slope in the simple regression of $\mathbf{y}$ on $\mathbf{z}_{1}%
-\hat{\mathbf{z}}_{1}$, the part of $\mathbf{z}_{1}$ that is uncorrelated to
the rest of explanatory variables. Also, $\mbox{Var}(\hat{\beta}_{z}%
)=\sigma^{2}/k$, the standard $t$-test statistic for the null hypothesis
$H_{0}:\beta_{z}=0$ is $t_{z}=\hat{\beta}_{z}\sqrt{k}/\hat{\sigma}$, where
$\hat{\sigma}^{2}= (\mathbf{y}_{1}-\hat{\mathbf{y}}_{1.X.z})^{T}%
(\mathbf{y}_{1}-\hat{\mathbf{y}}_{1.X.z})/(n_{1}-p-1)$, and the standard
$F$-test statistic for the same null hypothesis is
\[
F_{z}=t_{z}^{2}= \frac{\hat{\beta}_{z}^{2} k }{\hat{\sigma}^{2}}= \frac
{1}{\hat{\sigma}^{2}}\hat{\beta}_{z}(\mathbf{z}_{1}-\hat{\mathbf{z}}_{1}%
)^{T}(\mathbf{z}_{1}-\hat{\mathbf{z}}_{1})\hat{\beta}_{z}=
\]
\[
\frac{1}{\hat{\sigma}^{2}}(\hat{\mathbf{y}}_{1.X.z}-\hat{\mathbf{y}}%
_{1.X})^{T}(\hat{\mathbf{y}}_{1.X.z}-\hat{\mathbf{y}}_{1.X}) =\frac{n_{1}%
}{\hat{\sigma}^{2}}\mbox{Rel}_{\mbox{\scriptsize Om}}%
^{\mbox{\scriptsize Train}}(Z).
\]
The proof concludes when defining $\hat{\sigma}^{2}_{z.x,n_{1}}=
(\mathbf{z}_{1}-\hat{\mathbf{z}}_{1})^{T} (\mathbf{z}_{1}-\hat{\mathbf{z}}%
_{1})/n_{1}$.
\end{proof}

\begin{proposition}
\label{prop:Rel_Om} Assume that the regression function of $Y$ over $(X,Z)$ is
linear and that it is estimated by OLS. Then
\[
\frac{n_{1}}{\hat{\sigma}^{2}}\mbox{\rm Rel}_{\mbox{\scriptsize\rm Om}}%
(Z)=F_{z} \, \frac{\hat{\sigma}^{2}_{z.x,n_{1},n_{2}}}{\hat{\sigma}%
^{2}_{z.x,n_{1}}}= F_{z} \left(  1 + O_{p}\left(  \min\{n_{1},n_{2}%
\}^{-1/2}\right)  \right)  ,
\]
and
\[
\mbox{Rel}_{\mbox{\scriptsize Om}}(Z)= \hat{\beta}_{z}^{2} \hat{\sigma}%
^{2}_{z.x,n_{1},n_{2}},
\]
where $\hat{\sigma}^{2}_{z.x,n_{1},n_{2}}$ and $\hat{\sigma}^{2}_{z.x,n_{1}}$
are consistent estimators of the same parameter $\sigma^{2}_{z.x}$ (the
residual variance in the linear regression model $Z=X^{T}\alpha+\varepsilon
_{z}$), the first one depending on both, the training sample and the test
sample, and the second one (also appearing in Proposition \ref{prop:Rel_Om_Train})
only on the training sample.
\end{proposition}

\begin{proof}
The vectors of predicted values in the test sample are
$\hat{\mathbf{y}}_{2.X.z}=\mathbf{X}_{2}\hat{\beta}_{x}+\mathbf{z}_{2}
\hat{\beta}_{z}$ and, using equation (\ref{betayalfa}) in 
the proof of Proposition \ref{prop:Rel_Om_Train},
\[
\hat{\mathbf{y}}_{2.X}=\mathbf{X}_{2}\hat{\beta}_{0}=\mathbf{X}_{2}\left(
\hat{\beta}_{x}+\hat{\alpha}_{1}\hat{\beta}_{z}\right)  =\mathbf{X}_{2}%
\hat{\beta}_{x}+\hat{\mathbf{z}}_{2.1}\hat{\beta}_{z},
\]
where $\hat{\mathbf{z}}_{2.1}=\mathbf{X}_{2}\hat{\alpha}_{1}$ is the
prediction of $\mathbf{z}_{2}$ using the linear model fitted in the training
sample to predict $Z$ from $X$. Therefore, using the same notation as in
Proposition \ref{prop:Rel_Om_Train}, the \emph{relevance by omission} of the
variable $Z$ is
\[
\mbox{Rel}_{\mbox{\scriptsize Om}}(Z)=\frac{1}{n_{2}}(\hat{\mathbf{y}}%
_{2.X.z}-\hat{\mathbf{y}}_{2.X})^{T}(\hat{\mathbf{y}}_{2.X.z}-\hat{\mathbf{y}%
}_{2.X})=
\]
\[
\frac{1}{n_{2}}\hat{\beta}_{z}(\mathbf{z}_{2}-\hat{\mathbf{z}}_{2.1}%
)^{T}(\mathbf{z}_{2}-\hat{\mathbf{z}}_{2.1})\hat{\beta}_{z}= \frac{1}{n_{2}%
}\hat{\beta}_{z}^{2}k_{\mbox{\scriptsize Om}}=\frac{\hat{\sigma}^{2}F_{z}%
}{n_{1}}\frac{k_{\mbox{\scriptsize Om}}/n_{2}}{k/n_{1}},
\]
where $k_{\mbox{\scriptsize Om}}=(\mathbf{z}_{2}-\hat{\mathbf{z}}_{2.1}%
)^{T}(\mathbf{z}_{2}-\hat{\mathbf{z}}_{2.1})$ and $k$ has been defined in the
Appendix A.1 
of the paper. 
Observe that both,
$k_{\mbox{\scriptsize Om}}/n_{2}$ and $k/(n_{1}-p)$, are consistent estimators
of the residual variance in the linear regression model $Z=X^{T}%
\alpha+\varepsilon_{z}$. The proof concludes when defining $\hat{\sigma}%
^{2}_{z.x,n_{1},n_{2}}=k_{\mbox{\scriptsize Om}}/n_{2}$ and using $\hat
{\sigma}^{2}_{z.x,n_{1}}$ defined in Proposition \ref{prop:Rel_Om_Train}. The
expression involving the $O_{p}$ notation is derived by standard arguments for
the limit of a quotient.
\end{proof}

\newpage 

\subsubsection*{Supplement B: Additive model for the real data example}
The standard output of the additive model is shown in Table
\ref{tb:out.am.price}. Five variables are not significant at level size 0.001
(one of them, \texttt{floorLift}, enters in the model in a non-parametric
way). 
The adjusted $R^{2}$ is $0.781$ (slightly larger than in
the linear model, where it was $0.76$).

\begin{table}[ptb]
\caption{Rent housing prices: Standard output of the additive model.}%
\label{tb:out.am.price}
\begin{verbatim}
## log(price) ~ Barcelona + s(categ.distr, k = 3) + type.chalet + 
##     type.duplex + type.penthouse + type.studio + s(floor) + hasLift + 
##     s(floorLift, k = 6) + s(log.size) + exterior + s(rooms) + 
##     s(bathrooms, k = 6) + hasParkingSpace + ParkingInPrice + 
##     s(log_activation)
## 
## Parametric coefficients:
##                  Estimate Std. Error t value Pr(>|t|)    
## (Intercept)      7.276755   0.012134 599.690  < 2e-16 ***
## Barcelona        0.099034   0.005232  18.930  < 2e-16 ***
## type.chalet     -0.094697   0.020976  -4.514 6.41e-06 ***
## type.duplex     -0.024456   0.014702  -1.663  0.09626 .  
## type.penthouse   0.058281   0.010054   5.797 6.94e-09 ***
## type.studio      0.034965   0.033047   1.058  0.29006    
## hasLift          0.054562   0.012260   4.450 8.65e-06 ***
## exterior        -0.027271   0.006622  -4.118 3.85e-05 ***
## hasParkingSpace -0.004070   0.012444  -0.327  0.74363    
## ParkingInPrice  -0.038634   0.013312  -2.902  0.00371 ** 
## ---
## Signif. codes:  0 '***' 0.001 '**' 0.01 '*' 0.05 '.' 0.1 ' ' 1
## 
## Approximate significance of smooth terms:
##                     edf Ref.df       F p-value    
## s(categ.distr)    1.994  2.000 698.603  <2e-16 ***
## s(floor)          7.375  8.163  42.133  <2e-16 ***
## s(floorLift)      1.000  1.000   4.122  0.0423 *  
## s(log.size)       8.498  8.898 568.528  <2e-16 ***
## s(rooms)          7.305  7.957  43.673  <2e-16 ***
## s(bathrooms)      4.602  4.895 118.760  <2e-16 ***
## s(log_activation) 3.500  4.350 128.396  <2e-16 ***
## ---
## Signif. codes:  0 '***' 0.001 '**' 0.01 '*' 0.05 '.' 0.1 ' ' 1
## 
## R-sq.(adj) =  0.781   Deviance explained = 78.1%
## GCV = 0.064253  Scale est. = 0.064006  n = 11536
\end{verbatim}
\end{table}

Let us examine the relevance by ghost variables results (Figure
\ref{fig:VarRlevIdealista_gam_Gh}). Now, in the additive model, there is no
theoretical support for a direct relation between the $t$ and $F$-values shown
in Table \ref{tb:out.am.price}, and the relevance values plotted in the first
plot of Figure \ref{fig:VarRlevIdealista_gam_Gh}. In fact we can see that this
direct relation does not happen in this case (for instance, \texttt{log.size}
appears in Figure \ref{fig:VarRlevIdealista_gam_Gh} as much more relevant than
\texttt{categ.distr}, but the $F$-value of the term \texttt{s(categ.distr)} is
larger than that of \texttt{s(log.size)} in Table \ref{tb:out.am.price}). The
7 most relevant variables coincide with the 7 most statistically significant
in the additive model, that were also detected as the most relevant in the
linear model.

From the study of the relevance matrix $\mathbf{V} $ we observe that its first 7 eigenvectors 
are very similar to the corresponding ones in the linear model.
The first eigenvector is mainly associated with (\texttt{log.size}, while the second one is mainly related with \texttt{categ.distr}. Similarly, the sixth eigenvector is related with \texttt{floor}. 
On the other hand,
eigenvectors 3, 4, 5 and 7 seem to be jointly associated with the other 4
relevant variables, with the relationship between \texttt{bathrooms} and
\texttt{rooms} appearing stronger than that between \texttt{log\_activation}
and \texttt{Barcelona}.

\begin{figure}[p]
\begin{center}
\vspace*{-1cm} \hspace*{-1.5cm}
\includegraphics[scale=.45]{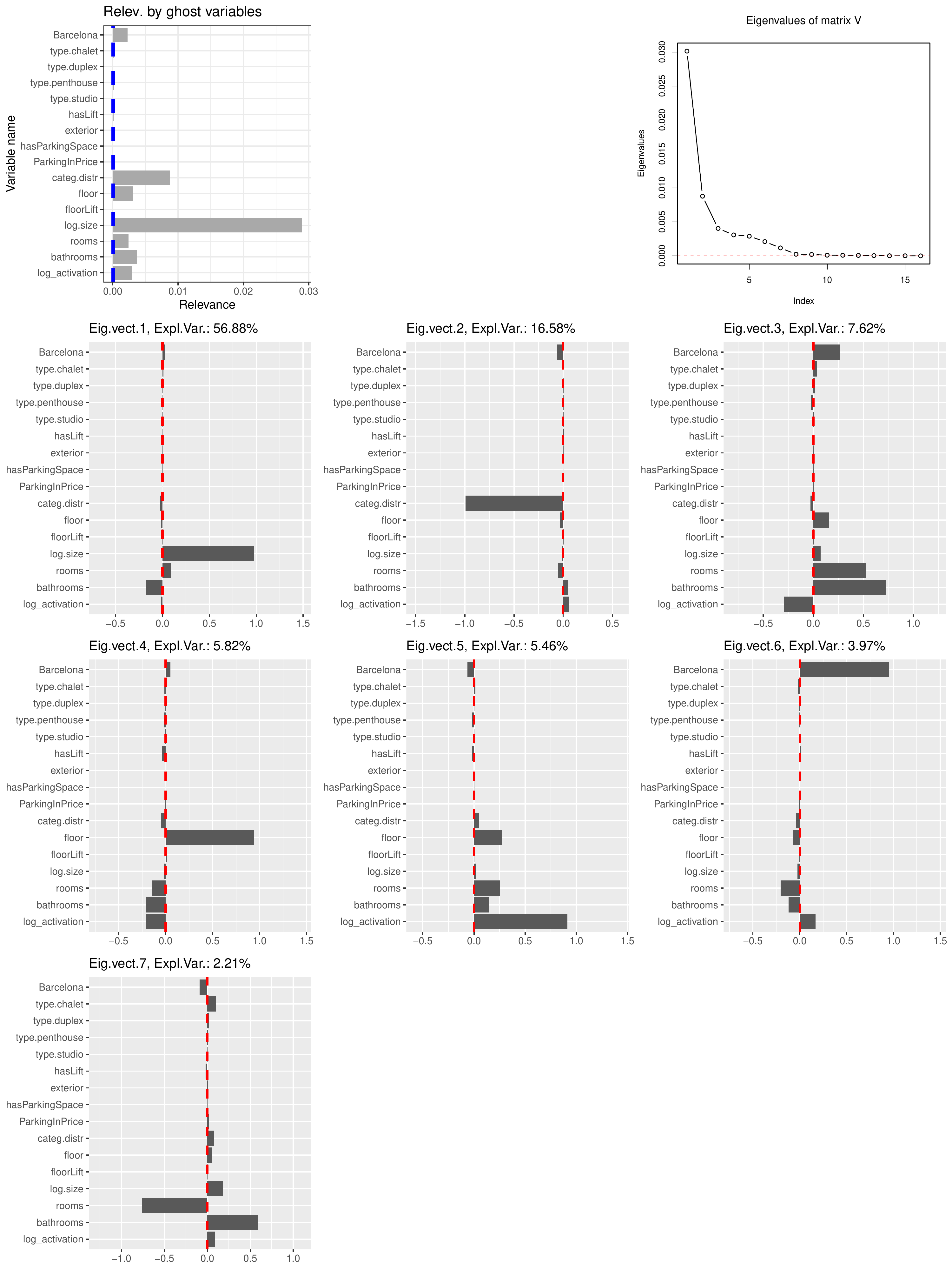}
\end{center}
\caption{Rent housing prices: Relevance by ghost variables for the additive
model. Only the eigenvectors explaining more than 1\% of the total relevance
are plotted. The order of the variables is the same as in the additive model
output in Table \ref{tb:out.am.price}.}%
\label{fig:VarRlevIdealista_gam_Gh}%
\end{figure}

Relevance by random permutations (Figure \ref{fig:VarRlevIdealista_gam_RP})
produces, again, results very similar to those corresponding to the linear model.
\texttt{log.size} is by far the most relevant variable, with
\texttt{bathrooms}, \texttt{categ.distr}, \texttt{rooms} and
\texttt{log\_activation} also having certain relevance. The eigen-structure of
the relevance matrix by random permutations $\tilde{ \mathbf{V} }$ shows that
there are 7 eigenvectors explaining more than 1\% of the total relevance, and
that each of them is mainly associated with one of the 7 most relevant
variables. 

\begin{figure}[p]
\begin{center}
\vspace*{-1cm} \hspace*{-1.5cm}
\includegraphics[scale=.45]{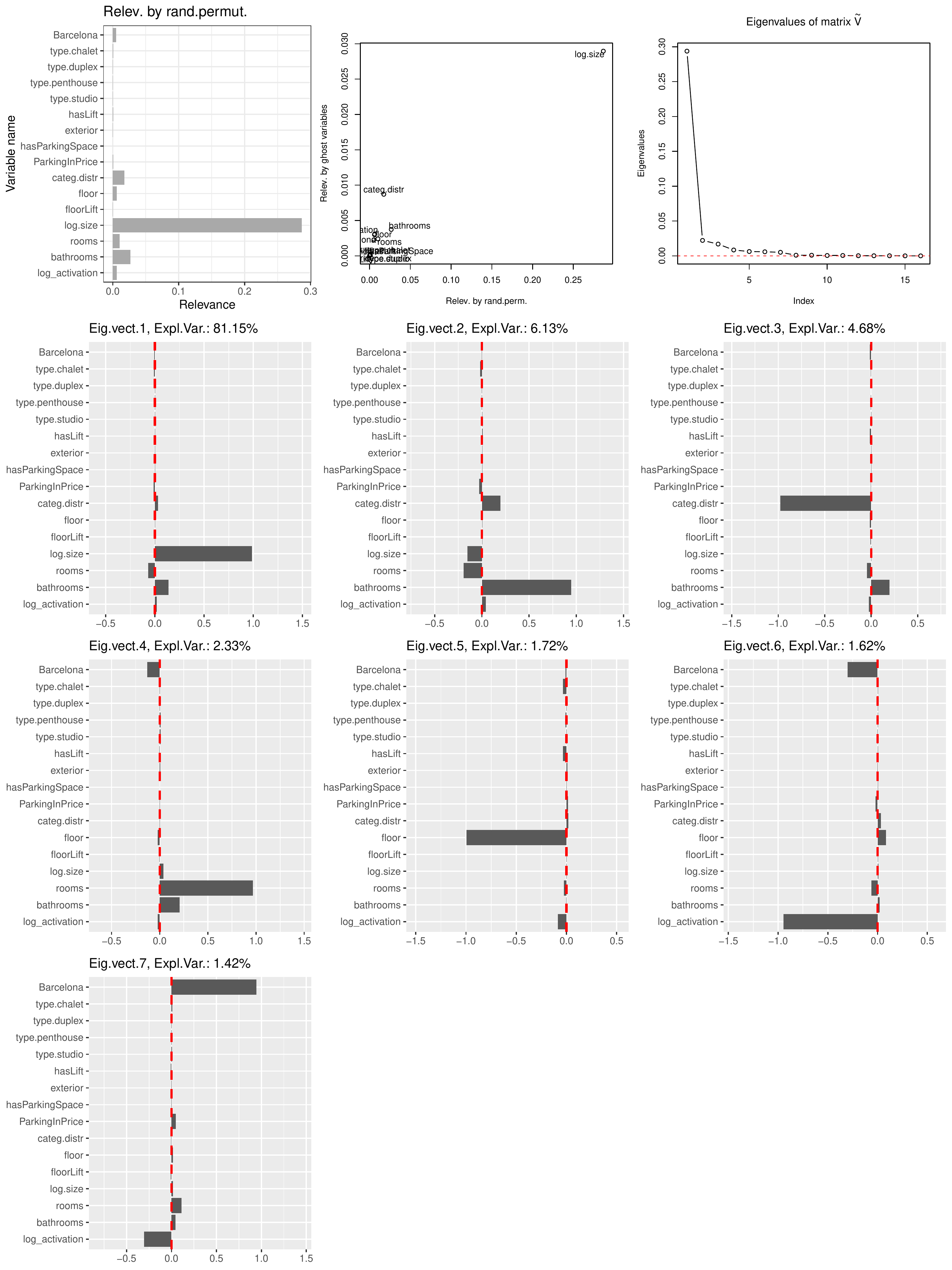}
\end{center}
\caption{Rent housing prices: Relevance by random permutations for the
additive model. Only the eigenvectors explaining more than 1\% of the total
relevance are plotted. The order of the variables is the same as in the
additive model output in Table \ref{tb:out.am.price}}%
\label{fig:VarRlevIdealista_gam_RP}%
\end{figure}

\end{document}